\documentclass{article}

\usepackage[preprint]{neurips_2020}
\usepackage[utf8]{inputenc} 
\usepackage[T1]{fontenc}    
\usepackage{url}            
\usepackage{booktabs}       
\usepackage{amsfonts}       
\usepackage{nicefrac}       
\usepackage{microtype}      
\usepackage{graphicx}
\usepackage{subfig}
\usepackage{amssymb,amsbsy,amsmath,amsfonts,amssymb,amscd}
\usepackage{amsthm}
\usepackage{comment}
\usepackage{verbatim}
\usepackage[dvipsnames]{xcolor}
\usepackage{enumitem}
\usepackage{algorithm,algpseudocode}
\usepackage{hyperref}       
\usepackage{setspace}
\newtheorem{theorem}{Theorem}[section]
\newtheorem{proposition}{Proposition}[section]
\newtheorem{definition}{Definition}[section]

\newtheorem{lemma}{Lemma}[section]
\newtheorem{assum}{Assumption}[section]
\newtheorem{cor}{Corollary}[section]

\newcommand{\EE}{\mathbb{E}}

\newcommand{\Cov}{\mathrm{Cov}}
\newcommand{\Var}{\mathrm{Var}}

\newcommand{\wx}{\tilde{x}}
\newcommand{\wv}{\tilde{v}}
\newcommand{\wrx}{\widetilde{\mathrm{X}}}
\newcommand{\wrv}{\widetilde{\mathrm{V}}}
\newcommand{\rx}{\mathrm{X}}

\newcommand{\wsx}{\widetilde{X}}
\newcommand{\wsv}{\widetilde{V}}

\newcommand{\rv}{\mathrm{V}}

\newcommand{\ww}{\tilde{w}}
\newcommand{\rd}{\,\mathrm{d}}

\title{Random Coordinate Underdamped Langevin Monte Carlo}

\author{%
  Zhiyan Ding\\
  Department of Mathematics\\
  University of Wisconsin-Madison\\
  Madison, WI 53706, USA \\
  \texttt{zding49@math.wisc.edu} \\
  \And
   Qin Li \\
  Department of Mathematics \\
  University of Wisconsin-Madison\\
  Madison, WI 53706, USA \\
  \texttt{qinli@math.wisc.edu} \\
  \And
  Jianfeng Lu\\
Mathematics Department\\
Duke University \\
Durham, NC 27708, USA \\
\texttt{jianfeng@math.duke.edu} \\
\AND
Stephen J. Wright\\
Computer Sciences Department\\
University of Wisconsin-Madison\\
Madison, WI 53706, USA \\
\texttt{swright@cs.wisc.edu}
}

\begin{document}

%

%
\maketitle
\begin{abstract}
The Underdamped Langevin Monte Carlo (ULMC) is a popular Markov chain Monte Carlo sampling method. It requires the computation of the full gradient of the log-density at each iteration, an expensive operation if the dimension of the problem is high. 
We propose a sampling method called Random Coordinate ULMC (RC-ULMC), which selects a single coordinate at each iteration to be updated and leaves the other coordinates untouched. 
We investigate the computational complexity of RC-ULMC and compare it with the classical ULMC for strongly log-concave probability distributions. We show that RC-ULMC is always cheaper than the classical ULMC, with a significant cost reduction when the problem is highly skewed and high dimensional. Our complexity bound for RC-ULMC is also tight in terms of dimension dependence. 
\end{abstract}

\section{Introduction}


Langevin Monte Carlo (LMC) is a popular Monte Carlo sampling method, widely used in Bayesian statistics and machine learning~\citep{MCMCforML}. 
The goal is to construct a Markov chain that approximately generates i.i.d. samples from a target distribution given by (with some abuse of notation, we do not distinguish a distribution with its density)
\begin{equation}\label{eqn:target distribution}
p_{\rx}(x)=\frac{1}{Z}e^{-f(x)}\,,
\end{equation}
where $Z$ is a normalizing constant that ensures $\int p_{\rx}(x)\rd x=1$. Throughout the paper we assume $f(x)$ is a convex function on $\mathbb{R}^d$, and thus $p_{\rx}(x)$ is a log-concave probability distribution.

Among the many Monte Carlo sampling methods, LMC~\citep{doi:10.1063/1.436415,PARISI1981378,roberts1996} stands out for its simplicity: For each iteration one updates the location of the particle by descending along the gradient and adding properly scaled Gaussian noise.
For strongly log-concave distributions, it has been established in recent years that the empirical distribution of the iterate in LMC converges exponentially fast to the target distribution, with total computational cost $\widetilde{O}(d^2/\epsilon^2)$ to achieve $\epsilon$ accuracy in Wasserstein distance~\citep{DALALYAN20195278,durmus2018analysis}. Here and throughout the paper, we measure ``cost" in terms of the total number of evaluations of a single element of the gradient, and assume that a full gradient evaluation requires about $d$ times as much computation as a single  component of the gradient.

To reduce the computational cost of sampling, the underdamped version of Langevin dynamics has recently been used to design the ULMC algorithm. 
By augmenting the state space with velocity variables, ULMC achieves faster convergence than LMC: To get $\epsilon$ accuracy, the computational complexity is $\widetilde{O}(d^{3/2} / \epsilon)$, improving dependence on both $d$ and $\epsilon$~\citep{Cheng2017UnderdampedLM,dalalyan2018sampling,eberle2018couplings}. 

This work aims at further improving the algorithm in terms of its dimension dependence, especially for the very high dimensional problems that arise often in practical applications (see \citep{ding2020random} for discussions of several examples). 
For these problems, ULMC requires a full evaluation of the gradient $\nabla f$ at each iteration, which often costs a factor of $O(d)$ greater than evaluating of a single component of the gradient. 
This factor arises when the expression of $\nabla f$ is not known explicitly, such as in  partial differential equation (PDE) based inverse problems, where $f$ is given implicitly by solving the forward problem given as a PDE, and finite-difference approximation to the full gradient would be $d$ times more expensive than a single component. While automatic differentiation techniques have been developed, the cost of evaluation of the gradient often still leads to formidable computational and memory cost. 

Other examples in which there is a factor-of-$d$ difference in evaluation cost between a full gradient and single component of the gradient come from problems with particular structures, such as graph-based problems. 
Given a graph with nodes ${\cal N} = \{1,2,\dotsc,d\}$ and directed edges ${\cal E} \subset \{ (i,j) : i,j \in {\cal N} \}$, suppose there is a scalar variable $x_i$ associated with each node $i=1,2,\dotsc,d$, and that the function $f$ has the form $f(x) = \sum_{(i,j) \in {\cal E}} f_{ij}(x_i,x_j)$. 
Then the partial derivative of $f$ with respect to $x_i$ is given by
\[
\frac{\partial f}{\partial x_i} = \sum_{j : (i,j) \in {\cal E}} \frac{\partial f_{ij}}{\partial x_i} (x_i,x_j) + \sum_{l : (l,i) \in {\cal E}} \frac{\partial f_{li}}{\partial x_i} (x_l,x_i)\,.
\]
Note that the number of terms in the summations in this expression equals the number of edges in the graph that touch node $i$, the expected value of which is about $2/d$ times the total number of edges in the graph. 
Meanwhile, evaluation of the full gradient would require evaluation of both partial derivatives of each component function $f_{ij}$ for {\em all} edges in the graph, leading to a factor-of-$d$ difference in evaluation cost.

In this work, we target these problems in which single components of the gradient are much less expensive than full gradients by incorporating the random coordinate descent (RCD) method from optimization into
underdamped Langevin sampling algorithm.
RCD  differs from gradient descent (GD) in that it updates just a single component, chosen at random, of the variable vector $x$ at each iteration. It takes a step in the negative gradient direction in just this component, leaving other components unchanged. (By contrast, gradient descent takes a step along the full negative gradient direction.) 
When there is a factor-of-$d$ difference in cost between evaluating the full gradient and a single component of the gradient, worst-case bounds for convex problems are better for RCD than for GD, the cost reduction being particularly significant when the dimension $d$ is high and $f$ is ``skewed'' in a sense to be defined later. 
Specifically, it was shown in~\cite{doi:10.1137/100802001} that when the coordinate is chosen from a distribution weighted according to the directional Lipschitz constants, the complexity is reduced from $d\kappa$ to $d\kappa_{\max}$, where $\kappa$ and $\kappa_{\max}$ are conditioning of $f$ and the maximum directional conditioning of $f$ respectively. 
Since $\kappa_{\max}\leq \kappa$ for all functions \citep{Ste-2015}, RCD is always cheaper than GD. 
Further, when the dimension of the problem is high and $f$ is skewed in the sense that $\kappa \approx d\kappa_{\max}$, the  reduction in cost approaches a factor of $d$.

In this paper, we propose the random coordinate underdamped Langevin Monte Carlo (RC-ULMC) algorithm.
We aim to improve the convergence of ULMC by utilizing cheaper steps, as in  RCD, so we establish non-asymptotic convergence rates for RC-ULMC and compare with classical ULMC. 
Our main results are as follows:
\begin{itemize}[itemsep=0pt]
\item[1.] The convergence rate of RC-ULMC depends on directional conditioning; see Theorem~\ref{thm:rculmc}.
\item[2.] Comparing with ULMC, RC-ULMC is always cheaper than the classical ULMC, the change being 
\[
d^{3/2}\kappa^{3/2} \to (d^{3/2}+\kappa)\kappa^{1/2}\,.
\]
This cost reduction is significant when $f$ is skewed and the dimension is high; see the discussion following Corollary~\ref{cor:rculmc}.
\item[3.] The complexity bound of the RC-ULMC we obtain is tight in both $d$ and $\epsilon$; see  Proposition~\ref{prop:badexampleW22}.
\end{itemize}

The remainder of the paper is as follows.
We review literature in Section~\ref{sec:references} and summarize basic notations and assumptions in Section~\ref{sec:notations}. 
In Section~\ref{sec:classical} we review  ULMC and its convergence properties. 
In Section~\ref{sec:algorithm} we present our new method RC-ULMC. 
Our main results are presented in Section~\ref{sec:result}, where we discuss the non-asymptotic convergence rate, the numerical cost, the cost saving compared to the classical ULMC, and the tightness of the result. 
Computational results are presented in  Section~\ref{sec:numerics}.
Technical derivations and proofs appear in Supplementary Materials.

\section{Related works}\label{sec:references}

The non-asymptotic analysis of LMC and ULMC sampling methods has been an active area \citep{Cheng2017UnderdampedLM,dalalyan2018sampling,DALALYAN20195278,durmus2018analysis}; and it has been established that ULMC gives a faster convergence under the same log-concavity and smoothness assumptions on the distribution. When ULMC is modified with a better discretization scheme, e.g., the random midpoint method, the computational complexity can be even further reduced \citep{Shen2019TheRM, YeH}. 

For ULMC, it was established in~\citep{dalalyan2018sampling} that it achieves $\epsilon$ error in Wasserstein metric within $\tilde{O}\left(d^{1/2}\kappa^{3/2}/\epsilon\right)$ iterations, where $\tilde{O}$ hides $\log$ factors.  The total cost of ULMC is therefore $\tilde{O}\left(d^{3/2}\kappa^{3/2}/\epsilon\right)$. In comparison, the RC-ULMC method proposed in this paper is always cheaper and the saving can be significant for highly skewed distributions in high dimension. 

The combination of RCD and LMC (based on overdamped Langevin dyanmics) has been recently explored in works \citep{shen2019nonasymptotic,ding2020random}. This algorithm will be referred to as RC-OLMC (where ``O" stands for overdamped). Compared with their result, the method in this paper converges faster both in terms of $d$ and $\epsilon$, similar to the saving obtained going from LMC to ULMC. This will be discussed further in~Section \ref{sec:result}.

Alternative sampling strategies have been developed without using the full gradient $\nabla f$ at each step. A standard approach is the Random Walk Metropolis algorithm, which combines a random walk proposal with Metropolis-Hastings acceptance-rejection step \citep{MCSH}, and thus only uses $f$ at each iteration. However, they are less efficient in high dimensions compared with gradient based methods \citep{mattingly2012diffusion,pillai2012optimal}. There have been recent interests in ensemble based sampling methods, in particular in the context of data assimilation, inspired by the ensemble Kalman filter \citep{Evensen:2006:DAE:1206873}, such as~\citep{EKS,Iglesias_2013}. Unfortunately, none of these methods can be completely ``gradient-free'' and at the same time consistent for non-Gaussian distributions~\citep{ding2019ensemble,ding2019ensemble2}. To achieve consistency, one can try to incorporate weights to particles, as is done in importance sampling \citep{IM1989} or sequential Monte Carlo \citep{SMCBOOK}, however such methods often face the difficulty of high variance~\citep{ding2020ensemble}. 

When the log-density $f(x)$ has the form of $f(x)=\sum^N_{i=1}f_i(x)$, one can randomly select a representative $\nabla f_r$ as a stochastic approximation to the full gradient, where $r$ is uniformly chosen from $\{1\,,\cdots\,,N\}$. This leads to the stochastic gradient Langevin Monte Carlo method~\citep{welling2011bayesian}. Note that in general we can write the full gradient as $\nabla f = \sum_{i=1}^d\partial_if \boldsymbol{e}_i$ (where $\boldsymbol{e}_i$ is the unit vector in $i$-th direction), and thus RCD and stochastic gradient, while used for different setups, share some similarity in reducing the cost of gradient evaluation.


\section{Notations and assumptions}\label{sec:notations}
Throughout the paper we assume convexity and gradient Lipschitz continuity of $f$.
\begin{assum}\label{assum:Cov}
The function $f$ is second-order differentiable and $\mu$-strongly convex for some $\mu>0$ and the gradient $\nabla f$ is $L$-Lipschitz. Specifically, we have: for all $x,x'\in\mathbb{R}^d$
\begin{equation}\label{Convexity}
f(x)-f(x')-\nabla f(x')^\top (x-x')\geq \frac{\mu}{2} |x-x'|^2\, 
\end{equation}
and
\begin{equation}\label{GradientLip}
|\nabla f(x)-\nabla f(x')|\leq L|x-x'|\,. 
\end{equation}
\end{assum}

Since the full gradient is Lipschitz continuous, so is its directional derivative. 
We denote directional Lipschitz constants by $L_i$, $i = 1, 2, \ldots, d$, meaning that
\begin{equation}\label{GradientLipcoord}
|\partial_i f(x+t\boldsymbol{e}_i)-\partial_i f(x)|\leq L_i |t|\,,
\end{equation}
for any $i=1,2,\dots,d$, any $x\in\mathbb{R}^d$, and any $t \in \mathbb{R}$. 

Denote $\nabla^2 f$ the Hessian, then the assumption implies that for all $x$
\[
\mu{I}_d\preceq \nabla^2 f(x) \preceq L{I}_d,\quad \left|\partial_{ii} f(x)\right|\leq L_i\,,
\]
where $\partial_{ii} f(x)$ is the $(i,i)$-element of $\nabla^2 f(x)$.
We also define condition numbers:
\begin{equation}\label{eqn:R}
\kappa=L/\mu\geq 1,\  \kappa_i=L_i/\mu\geq1,\  \kappa_{\max} = \max_i\kappa_i\,.
\end{equation}
As shown in~\cite{Ste-2015}, we have
\begin{equation} \label{eq:LiL}
\kappa_i \le \kappa_{\max} \le \kappa\le d\kappa_{\max}\,.
\end{equation}
We note that both inequalities, $\kappa_{\max} \leq \kappa$ and $\kappa\leq d\kappa_{\max}$ are sharp. 
If $\nabla^2 f$ is a diagonal matrix, then $L_{\max}=L$, both being the largest eigenvalue of $\nabla^2 f$, so that $\kappa_{\max} = \kappa$. 
This is the case when all coordinates are independent of each other, for example $f = \sum_i\lambda_ix_i^2$. 
On the other hand, if $f$ is highly skewed, such as  $f = (\sum_i x_i)^2$, so $\nabla^2f = \mathsf{e}\cdot\mathsf{e}^\top$ (where $\mathsf{e} = [1,1,\dots,1]^\top$), then $L = dL_{\max}$ and $\kappa = d\kappa_{\max}$.

Furthermore, since $\kappa_i\geq 1$, we have for $p>0$ that
\[
   (d-1)+\kappa^p_{\max}\leq \sum^d_{i=1} \kappa^p_i\leq d\kappa^p_{\max}\leq d\kappa^p\,. 
\]
Both bounds are tight. In the case when $f=\kappa_1|x_1|^2+\sum^d_{i=1}|x_2|^2$ with $\kappa_1>1$, $\sum^d_{i=1} \kappa^p_i=(d-1)+\kappa^p_{1}=(d-1)+\kappa^p_{\max}$. On the other hand, when $f=\sum_ix^2_i$, we have $\kappa_i=\kappa_{\max}=1$, then $\sum^d_{i=1} \kappa^p_i=d\kappa^p_{\max} = d\kappa^p$. And we say $f$ is highly skewed if 
\begin{equation}\label{fschew}
\sum^d_{i=1} \kappa^p_i\approx (d-1)+\kappa^p_{\max}\,.
\end{equation}

To measure the distance between two probability distributions, we use the Wasserstein distance.
\begin{definition}\label{def:W}
The Wasserstein distance $W_p$ (for any $p \ge 1$) between probability measures $\mu$ and $\nu$ is defined as
\[
W_p(\mu,\nu) = \left(\inf_{(X,Y)\in \Gamma(\mu,\nu)} \mathbb{E}|X -Y|^p\right)^{1/p}\,,
\]
where $\Gamma(\mu,\nu)$ is the set of distribution of $(X,Y)\in\mathbb{R}^{2d}$ whose marginal distributions, for $X$ and $Y$ respectively, are $\mu$ and $\nu$.
\end{definition}

In this paper, we will use the $2$-Wasserstein metric $W_2$.

\section{Classical ULMC}\label{sec:classical}
Underdamped Langevin dynamics is characterized by the following SDE:
\begin{equation}\label{eqn:underdampedLangevin}
\left\{\begin{aligned}
&\rd X_t = V_t\rd t\,,\\
&\rd V_t = -2 V_t\rd t-\gamma\nabla f(X_t)\rd t+\sqrt{4 \gamma}\rd B_t\,, 
\end{aligned}\right.
\end{equation}
where $\gamma>0$ is a parameter to be tuned, and $B_t$ is the Brownian motion. 
Here we use the parametrization form of \cite{Cheng2017UnderdampedLM} (alternative parametrizations are used in \cite{dalalyan2018sampling,Shen2019TheRM}). 
Denoting by $q(x,v,t)$ the probability density function of $(X_t,V_t)$, we have that $q$ satisfies the Fokker-Planck equation
\[
\partial_tq=\nabla\cdot \left(\begin{bmatrix}
-v\\
2v+\gamma\nabla f
\end{bmatrix}q+\begin{bmatrix}
0 & 0\\
0 & 2\gamma
\end{bmatrix}\nabla q\right)\,.
\]
It is well known that under mild conditions, $q$ converges to $p(x,v)\propto\exp(-(f(x)+|v|^2/2\gamma))$ (see e.g., \cite{Villani2006,DOLBEAULT2009511,Baudoin2016WassersteinCP,cao2019explicit}), and thus the marginal density function for $x$ becomes the target distribution $p_{\rx}(x)$.

Denoting by $h>0$ the time stepsize, we have that for $t\in[mh,(m+1)h]$, ~\eqref{eqn:underdampedLangevin} is equivalent to
\begin{equation}\label{eqn:langevin_integral}
\begin{aligned}
X(t)=&X(mh)+\frac{1-e^{-2t}}{2}V(mh)-\frac{\gamma}{2}\int^t_{mh}\left(1-e^{-2(t-s)}\right)\nabla f\left(X(s)\right)\rd s\\
&+\sqrt{\gamma}\int^t_{mh}\left(1-e^{-2(t-s)}\right)\rd B_s\,,
\end{aligned}
\end{equation}
\begin{equation}\label{eqn:langevin_integral_v}
\begin{aligned}
V(t)=&V(mh)e^{-2t}-\gamma\int^t_{mh}e^{-2(t-s)}\nabla f\left(X(s)\right)\rd s+\sqrt{4 \gamma}e^{-2 t}\int^t_{mh}e^{2 s}\rd B_s\,.
\end{aligned}
\end{equation}
The sampling method, ULMC, can be viewed as a numerical solver for  \eqref{eqn:langevin_integral}-\eqref{eqn:langevin_integral_v} based on the Euler approximation. 
Denoting by $(x^m,v^m)$ the numerical approximation to $(X(mh),V(mh))$, and replacing $X(s)$ in~\eqref{eqn:langevin_integral}-\eqref{eqn:langevin_integral_v} by $x^m$,  Euler approximation yields that $(x^{m+1},v^{m+1})\in\mathbb{R}^{2d}$ are two Gaussian random vectors with the following expectation and covariance:
\begin{align}
&\EE, x^{m+1}=x^m+\frac{1}{2}\left(1-e^{-2h}\right)v^m\nonumber-\frac{\gamma}{2}\left(h-\frac{1}{2}\left(1-e^{-2h}\right)\right)\nabla f(x^m)\,,\nonumber\\
&\EE \, v^{m+1}=v^me^{-2h}-\frac{\gamma}{2}\left(1-e^{-2h}\right)\nabla f(x^m)\,,\nonumber\\
&\Cov\left(x^{m+1}\right)=\gamma\left[h-\frac{3}{4}-\frac{1}{4}e^{-4h}+e^{-2h}\right]\cdot I_d\,,\nonumber\\
&\Cov\left(v^{m+1}\right)=\gamma\left[1-e^{-4h}\right]\cdot I_d\,,\nonumber\\
&\Cov\left(x^{m+1}\,,v^{m+1}\right)=\frac{\gamma}{2}\left[1+e^{-4h}-2e^{-2h}\right]\cdot I_d\,\label{distributionofZ}.
\end{align}
Here $\EE$ denotes the expectation, and $\Cov(a,b)$ denotes the covariance of $a$ and $b$ (abbreviated to $\Cov(a)$ when $b=a$), and  $I_d$ is the identity matrix in $\mathbb{R}^d$. 
We thus draw $x^{m+1},v^{m+1}$ from this Gaussian distribution numerically to update the iteration. 
We summarize ULMC in Algorithm~\ref{alg:OULMC}.
\begin{algorithm}[htb]
\caption{\textbf{Underdamped Langevin Monte Carlo (ULMC)}}\label{alg:OULMC}
\begin{algorithmic}
\State {\bf Input:} $h$ (time stepsize); $\gamma$ (parameter); $d$ (dimension); $\nabla f(x)$; $M$ (stopping index).
\State {\bf Initial:} $(x^0,v^0)$ i.i.d. sampled from the initial distribution $q_0(x,v)$.
\For{$m=0\,,1\,,\cdots,M$}
\State 1. Compute the expectation and the covariance as in~\eqref{distributionofZ}.
\State 2. Sample $(x^{m+1},v^{m+1})$ from the associated Gaussian distribution.
\EndFor
\State \textbf{Output:} $\{x^m\}$.
\end{algorithmic}
\end{algorithm}

The algorithm converges exponentially when $f$ is strongly convex with Lipschitz continuous gradient; see~\cite{dalalyan2018sampling}. 
The original statement uses a different parametrization. We translate the result to the current one in Supp.~\ref{sec:proofofthm:convergenceulmc1} and restate the result here.
\begin{theorem}\label{thm:convergenceulmc}[\cite[Theorem~2]{dalalyan2018sampling}]
Assume $f$ satisfies Assumption~\ref{assum:Cov} and  that 
\[
\gamma\leq \frac{4}{\mu+L}, \quad h\leq\frac{\gamma^{1/2}\mu}{8L}.
\]
Then we have
\begin{equation}\label{eqn:convergenceulmc}
W_m\leq \sqrt{2}\exp(-0.375\mu h\gamma^{1/2} m)W_0+(2d)^{1/2}\kappa h\,.
\end{equation}
Here $W_m := W_2(q_m,p)$ and $q_m(x,v)$ denotes the probability density function of iteration $m$ of ULMC.
Moreover, suppose the initial $W_0$ is $O(1)$, then the total number of iterations to achieve $\epsilon$ accuracy is $\tilde{O}\left(\frac{d^{1/2}\kappa^{3/2}}{\mu^{1/2}\epsilon}\right)$, and the cost is $\tilde{O}\left(\frac{d^{3/2}\kappa^{3/2}}{\mu^{1/2}\epsilon}\right)$.
\end{theorem}
The cost depends on both the dimensionality $d$ and condition number $\kappa$ with $3/2$ power for both.

\section{Randomized Coordinate Underdamped Langevin Monte Carlo}\label{sec:algorithm}



We integrate the RCD idea into ULMC to yield our method RC-ULMC. Instead of updating every entry of the process as is done in~\eqref{eqn:langevin_integral}--\eqref{eqn:langevin_integral_v}, we randomly select one direction $r^m \in \{1,2,\dotsc,d\}$ and evolve only $(X_{r^m}\,,V_{r^m})(t)$. 
Correspondingly, we would only change one single entry  $(x^m_{r^m},v^m_{r^m})$ according to expectation and covariance, analogous to~\eqref{distributionofZ}.

We denote the discrete distribution from which $r^m$ is chosen by $\Phi$, with $\phi_i$ being the probability of component $i$ being chosen, that is,
\begin{equation} \label{eq:def.Phi}
\Phi := \{\phi_1,\phi_2,\dotsc, \phi_d\}\,,
\end{equation}
where $\phi_i > 0$ for all $i$ and $\sum_{i=1}^d \phi_i=1$. 
Denoting by $h_i$ the stepsize when $i$-th direction is chosen, we choose $h_i$ to be inversely proportional to $\phi_i$, as follows:
\begin{equation}\label{condition:pranh}
h_i=\frac{h}{\phi_i}\,, \quad i=1,2,\dotsc,d\,,
\end{equation}
where $h>0$ is a parameter that can be viewed as the expected stepsize. 
We also define the total elapsed time after $m$ steps as
\[
T^m=\sum^{m-1}_{n=0} h_{r^n}\,.
\]
The initial iterate $(x^0,v^0)$ is drawn from a distribution $q_0$, which  can be any distribution that is easy to draw from (e.g., a normal distribution).

Because only component $r^m$ is updated at iteration $m$ of  RC-ULMC, we have for  $t \in [T^m, T^{m+1}]$ that 
\begin{align}
X_{r^m}(t)&=X_{r^m}(T^m)+\frac{1-e^{-2t}}{2}V_{r^m}(T^m)-\frac{\gamma}{2}\int^t_{T^m}\left(1-e^{-2(t-s)}\right)\partial_{r^m} f(X(s))\rd s\nonumber\\
&\quad +\sqrt{\gamma}\int^t_{T^m}\left(1-e^{-2(t-s)}\right)\rd B_s\,,\label{eqn:ULDSDE2continummx}\\
V_{r^m}(t)&=V_{r^m}(T^m)e^{-2(t-T^m)}-\gamma\int^{t}_{T^m}e^{-2(t-s)}\partial_{r^m}f(X(s)) \rd s+\sqrt{4\gamma}\int^{t}_{T^m}e^{-2(t-s)}\rd B_s\,,\label{eqn:ULDSDE2continummv}\\
X_{i}(t)& =X_i(T^m),\; V_{i}(t)=V_i(T^m)\,,\quad i\neq r^m\label{eqn:ULDSDE2continummxvfixed}\,.
\end{align}
%
To obtain a practical algorithm, we apply the Euler approximation to these dynamics. 
Denoting by $(x^{m}_{r^m},v^{m}_{r^m})$ the numerical approximation to $\left(X(T^m)\,,V(T^m)\right)$, we replace $\partial_{r^m} f(X(s))$ in~\eqref{eqn:ULDSDE2continummx}-\eqref{eqn:ULDSDE2continummxvfixed} by $\partial_{r^m} f\left(x^m\right)$, so that  $x^{m+1}_i=x^m_i$, $v^{m+1}_i=v^m_i$ for $i\neq r^m$, and $(x^{m+1}_{r^m}\,,v^{m+1}_{r^m})$ are two Gaussian random variables with the following expectation and covariance:
\begin{align}
&\EE x^{m+1}_{r^m}=x^m_{r^m}+\frac{1}{2}\left(1-e^{-2h_{r^m}}\right)v^m_{r^m}-\frac{\gamma}{2}\left(h_{r^m}-\frac{1}{2}\left(1-e^{-2h_{r^m}}\right)\right)\partial_{r^m} f(x^m)\,,\nonumber\\
&\EE v^{m+1}_{r^m}=v^m_{r^m}e^{-2h_{r^m}}-\frac{\gamma}{2}\left(1-e^{-2h_{r^m}}\right)\partial_{r^m} f(x^m)\,,\nonumber\\
&\mathrm{Cov}\left(x^{m+1}_{r^m}\right)=\gamma\left[h_{r^m}-\frac{3}{4}-\frac{1}{4}e^{-4h_{r^m}}+e^{-2h_{r^m}}\right]\,,\nonumber\\
&\mathrm{Cov}\left(v^{m+1}_{r^m}\right)=\gamma\left[1-e^{-4h_{r^m}}\right]\,,\nonumber\\
&\Cov\left(x^{m+1}_{r^m}\,,v^{m+1}_{r^m}\right)=\frac{\gamma}{2}\left[1+e^{-4h_{r^m}}-2e^{-2h_{r^m}}\right]\,\label{distributionofZ3}.
\end{align}
Then, $(x^{m+1}_{r^m}\,,v^{m+1}_{r^m})$ is drawn according to this Gaussian distribution for the update. 
We summarize the RC-ULMC approach in Algorithm~\ref{alg:RCD-OULMC}.

\begin{algorithm}[htb]
\caption{\textbf{Random Coordinate Underdamped Langevin Monte Carlo} (RC-ULMC)}\label{alg:RCD-OULMC}
\begin{algorithmic}
\State {\bf Input:} $h$ (time stepsize); $\gamma$ (parameter); $d$ (dimension); $\nabla f(x)$; probability set $\Phi := \{\phi_1,\phi_2,\dotsc, \phi_d\}$; $M$ (stopping index).
\State {\bf Initial:} $(x^0,v^0)$ i.i.d. sampled from the initial distribution induced by $q_0(x,v)$.
\For{$m=0\,,1\,,\cdots,M$}
\State 1. Draw $r$ randomly from $1,2,\dotsc,d$ according to $\Phi$;
\State 2. Update $(x^{m+1},v^{m+1})$ as follows:
\begin{itemize}
\item $x^{m+1}_i=x^m_i,v^{m+1}_i=v^{m}_i$ for $i\neq r$;
\item sample $(x^{m+1}_r,v^{m+1}_r)$ as Gaussian variables according to~\eqref{distributionofZ3}.
\end{itemize}
\EndFor
\State \textbf{Output:} $\{x^m\}$.
\end{algorithmic}
\end{algorithm}

%
%
%

\section{Main results}\label{sec:result}
We have three main results regarding the underlying dynamics (the SDE), and the RC-ULMC algorithm.
In Section~\ref{sec:continuumofrculmc}, we discuss convergence  of the SDE~\eqref{eqn:ULDSDE2continummx}-\eqref{eqn:ULDSDE2continummxvfixed}. 
This SDE can be viewed as  the continuum version of the RC-ULMC algorithm. 
Only with the convergence of this SDE can we hope for the convergence of RC-ULMC.
In Section~\ref{sec:rculmc}, we describe the non-asymptotic convergence properties of RC-ULMC. 
From this result, we can determine an optimal strategy for selecting the coordinate $r^m$ at each iteration.
We will also compare our results with those for classical ULMC, showing that the bounds for RC-ULMC are always better. 
Moreover, when $f$ is highly skewed --- for example when $\kappa_1=\kappa_{\max}\gg 1$ and $\kappa_i\approx 1$ for $i\geq2$ --- the total cost is $\widetilde{O}\left(\left(d^{3/2}+\kappa_{\max}\right)\kappa^{1/2}/\epsilon\right)$, as compared to $\widetilde{O}\left(d^{3/2}\kappa^{3/2}/\epsilon\right)$ for ULMC.

We provide an example in Section~\ref{sec:tight} to show that our bounds are tight.



\subsection{Convergence of SDEs}\label{sec:continuumofrculmc}
Our first result is on the convergence of the SDE~\eqref{eqn:ULDSDE2continummx}-\eqref{eqn:ULDSDE2continummxvfixed}, the underdamped Langevin dynamics that incorporates random coordinate selection.

Denote $X^m=X(T^m)$, $V^m=V(T^m)$ and denote the probability filtration by $\mathcal{F}^m=\left\{ x^0, v^0, r^{n\leq m}, B_{s\leq T^m}\right\}$. Then we have the following result about its geometric ergodicity.
\begin{theorem}\label{thm:contiummRCULMC}
Suppose that $f$ satisfies Assumption~\ref{assum:Cov} and
\[
\gamma\leq \frac{1}{L},\quad h\leq \frac{\gamma \mu \min\{\phi_i\}}{312+12\gamma+8L+432L^2}\,,
\]
then $\left\{(X^m,V^m)\right\}^\infty_{m=0}$ is a Markov chain. 
Denoting by $q_m(x,v)$ the probability density function of $(X^m,V^m)$, we have the following:
\begin{itemize}
\item The stationary distribution has density $p(x,v)\propto\exp(-(f(x)+|v|^2/2\gamma))$.
\item When the initial distribution $q_0$ has finite second moments, there exist constants $R>0$ and $r>1$ independent of $m$ such that
\begin{equation}\label{eqn:converge3}
\int_{\mathbb{R}^{2d}} |q_m(x,v)-p(x,v)|\rd x\rd v\leq Rr^{-m}\,.
\end{equation}
\end{itemize}
\end{theorem}
The proof, which  can be found in Supp.~\ref{sec:proofofthm:convergenceulmc}, uses the convergence analysis framework of \citep{MATTINGLY2002185}, based on construction of a special Lyapunov function. 
This theorem suggests that the TV distance between $q_m$ and $p$ decays exponentially, meaning that $(X^m,V^m)$ can be seen to be drawn from the target distribution $p$ as $m\to\infty$. 
Since the RC-ULMC algorithm is its Euler approximation, the samples generated by this algorithm are drawn from $p$ as well --- approximately, up to a discretization error.

Note that $R$ and $r$ are independent of $m$ in Theorem~\ref{thm:contiummRCULMC}, but we do not have explicit control on its dependence on parameters such as $h$, $d$, and $L$. 
This is worse in comparison with the results in~\citep{Cheng2017UnderdampedLM,cao2019explicit} for the underdamped Langevin dynamics, where the convergence rate is characterized explicitly in terms of all parameters. 
The difficulty of our case comes mainly from the complicated process of coordinate selection, which prevents us from applying the synchronous coupling approach of \cite{Cheng2017UnderdampedLM, DALALYAN20195278} directly to the dynamics~\eqref{eqn:ULDSDE2continummx}--\eqref{eqn:ULDSDE2continummxvfixed} to establish contraction.
Whether the hypocoercity estimate of \cite{cao2019explicit} can be applied remains an interesting future research direction.

\subsection{Convergence of RC-ULMC}\label{sec:rculmc} 
Regarding the non-asymptotic error analysis of RC-ULMC, we have the following result (cf.~Theorem~\ref{thm:convergenceulmc}).
\begin{theorem}\label{thm:rculmc}
Suppose that $f$ satisfies Assumption~\ref{assum:Cov} and that
\begin{equation}\label{eqn:cond_dis}
\gamma\leq \frac{1}{L},\quad h\leq\min\left\{\frac{\gamma\mu \min\{\phi_i\}}{240}\right\}\,.
\end{equation}
Denote $q_m(x,v)$ the probability density function of iteration $m$ of RC-ULMC and define $W_m=W_2(q_m,p)$.
Then we have
\begin{equation}\label{thm:iterationulmc}
W_m\leq 4\exp\left(-\frac{\mu\gamma mh}{8}\right)W_0+40\gamma^{1/2} h\sqrt{\sum^d_{i=1}\frac{\kappa^2_i}{\phi^2_i}}\,.
\end{equation}
\end{theorem}
The proof can be found in Supp.~\ref{sec:proofofthm:rculmc}. 
This result indicates that the Wasserstein distance between $q_m$ and the target distribution decays exponentially except for an error term of size $O(h)$. 
The convergence rate is given by $\mu\gamma$, and with the choice $\gamma=1/L$, this quantity is  the inverse  condition number $1/\kappa$ of the objective function (see \eqref{eqn:R}). 
The second term in~\eqref{thm:iterationulmc} reflects the discretization error, with its size being determined by the directional condition number $\kappa_i$ (see \eqref{eqn:R}) and the random selection probability distribution $\Phi$.

This theorem not only  allows us to estimate the number of iterations required to achieve a preset accuracy, but also suggests that we choose $\{\phi_i\}$ in a way that minimizes the bound. %
\begin{cor}\label{cor:rculmc}
Suppose that the conditions of Theorem \ref{thm:rculmc} hold and $\gamma=1/L$. We have the following estimates. 
\begin{itemize}
\item For any $\epsilon>0$, the number of needed iterations $M$ to attain $W_M\leq \epsilon$ is
\begin{equation}\label{eqn:boundM}
M=\Theta\left(\frac{\kappa^{1/2}\sqrt{\sum^d_{i=1}\kappa^2_i/\phi^2_i}}{\mu^{1/2}\epsilon}\log \left(\frac{W_0}{\epsilon}\right) \right)\,.
\end{equation}

\item The optimal choice of $\phi_i$ is
\begin{equation}\label{eqn:phiLi}
\phi_i=\frac{L^{2/3}_i}{\sum^d_{j=1}L^{2/3}_j},\quad i=1,2,\dotsc,d.
\end{equation}
In this case, the number of iterations required is 
\begin{equation}\label{eqn:boundM2}
M=\Theta\left(\frac{\kappa^{1/2}\left(\sum^d_{j=1}\kappa^{2/3}_j\right)^{3/2}}{\mu^{1/2}\epsilon}\log \left(\frac{W_0}{\epsilon}\right) \right)\,.
\end{equation}
\end{itemize}
\end{cor}
\begin{proof}
According to Theorem \ref{thm:rculmc}, we achieve $W_m \le \epsilon$ by ensuring that  both terms in~\eqref{thm:iterationulmc} are less than $\epsilon/2$. 
Thus, in addition to \eqref{eqn:cond_dis}, we need
\begin{equation}\label{boundh}
h\leq\frac{\epsilon L^{1/2}}{80\sqrt{\sum^d_{i=1}\frac{\kappa^2_i}{\phi^2_i}}},\quad m\geq \frac{8\kappa}{h}\log \frac{8W_0}{\epsilon}\,.
\end{equation}
We thus obtain~\eqref{eqn:boundM}. Finding the optimal choice of $\phi_i$ amounts to minimizing the error term:
\[
\min_{\Phi} \, \sum^d_{i=1}\frac{\kappa^2_i}{\phi^2_i}\,,\quad\text{s.t.}\quad \sum^d_i\phi_i=1\,,\phi_i>0\,.
\]
An elementary argument based on constrained optimization theory leads to \eqref{eqn:phiLi}.
\eqref{eqn:boundM2} is then obtained by substituting \eqref{eqn:phiLi} into \eqref{eqn:boundM}.
\end{proof}

Suppose the objective function $f$ is well-conditioned in every direction, so that $\kappa_i = O(1)$ for all $i$. 
Then according to both~\eqref{eqn:boundM} and~\eqref{eqn:boundM2}, we see that the cost is roughly $\widetilde{O}(d^{3/2}/\epsilon)$. 
This order is the same as for the classical ULMC shown in Theorem~\ref{thm:convergenceulmc}. 

When $f$ is not as well-conditioned, meaning $\kappa_i$ are not uniformly small in every direction, then RC-ULMC can have a significant advantage over the classical ULMC. In practice, if we have some a priori estimate of $\kappa_i$, we can choose the optimal $\Phi$ and the cost estimate will be given by~\eqref{eqn:boundM2}. Of course, such a priori information might not be available, in such case, we can choose uniformly the coordinate at each iteration: $\phi_i = 1/d$. We compare below the cost of RC-ULMC in these two scenarios with  $\widetilde{O}(d^{3/2}\kappa^{3/2}/\mu^{1/2} \epsilon)$, the cost of the classical ULMC, as shown in Theorem~\ref{thm:convergenceulmc}.
\begin{enumerate}[wide, labelwidth=!, labelindent=0pt]
\item[Case 1:] Uniform sampling, where we choose $\phi_i=1/d$. From \eqref{eqn:boundM}, we found that the cost is 
\begin{equation}\label{eqn:cost_RC-LMC}
\widetilde{O}\left(\frac{d\kappa^{1/2}\sqrt{\sum^d_{i=1}\kappa^2_i}}{\mu^{1/2}\epsilon}\right)\,,
\end{equation}
where the $\widetilde{O}$ ignores log terms. Since $\sum^d_{i=1}\kappa^2_i\leq d\kappa_{\max}^2\leq d\kappa^2$, we observe that RC-ULMC is always cheaper than ULMC. Furthermore, when $f$ is highly skewed \eqref{fschew} with $p=2$, then~\eqref{eqn:cost_RC-LMC} is reduced to 
\[
\tilde{O}\left(\frac{d(\sqrt{d}+\kappa_{\max})\kappa^{1/2}}{\mu^{1/2}\epsilon}\right)\,.
\]
This bound indicates that RC-ULMC is significantly cheaper than  ULMC when both $d$ and $\kappa$ are large.
\item[Case 2:] With the optimal choice of $\Phi$ (using~\eqref{eqn:phiLi}), the cost of RC-ULMC is equivalent to $M$ in~\eqref{eqn:boundM2}. 
We still have that RC-ULMC is always cheaper than ULMC. 
Furthermore, when $f$ is highly skewed \eqref{fschew} with $p=3/2$, then we have upper bound 
\[
\left( \sum^d_{j=1}\kappa^{2/3}_j \right)^{3/2}\approx \left((d-1)+\kappa^{2/3}_{\max}\right)^{3/2} \leq 2(d^{3/2}+\kappa_{\max})\,.
\]
By substituting this bound into \eqref{eqn:boundM2},  we get the following bound:
\begin{equation}\label{rate:optimal}
    \tilde{O}\left(\frac{(d^{3/2}+\kappa_{\max})\kappa^{1/2}}{\mu^{1/2}\epsilon}\right)\,.
\end{equation}
The reduction over ULMC is significant when either $d$ or $\kappa$ is large. 
\end{enumerate}

We also compare RC-ULMC with RC-OLMC discussed in~\citep{shen2019nonasymptotic,ding2020random}. To achieve $\epsilon$-accuracy, the total cost of RC-OLMC is
\[
\widetilde{O}\left(\frac{\sum^d_{i=1}\kappa^2_i/\phi_i}{\mu\epsilon^2}\right)\,.
\]
Noting that $\sqrt{\sum^d_{i=1}\kappa^2_i/\phi^2_i}\leq \sum^d_{i=1}\kappa_i/\phi_i\leq \sum^d_{i=1}\kappa^2_i/\phi_i$,~\eqref{eqn:boundM} is always smaller, meaning RC-ULMC is always cheaper than RC-OLMC when $\epsilon<1/L^{1/2}$. Furthermore, if we choose uniform sampling $(\phi_i=1/d)$ and assume $\gamma,\mu,\kappa_i$ are $O(1)$, then the cost of RC-ULMC is $\widetilde{O}\left(d^{3/2}/\epsilon\right)$ while the cost of RC-OLMC is $\widetilde{O}\left(d^2/\epsilon^2\right)$. We have a significant improvement in both $d$ and $\epsilon$.

Finally, we note that in \citep{Shen2019TheRM}, randomzied midpoint method (RMM) is used to discretize SDE \eqref{eqn:underdampedLangevin}. 
According to \citep{YeH},  RMM needs $\widetilde{O}(d^{1/3}\kappa/\epsilon^{2/3})$ iteration steps to achieve $\epsilon$-accuracy, which equates to a cost of $\widetilde{O}(d^{4/3}\kappa/\epsilon^{2/3})$.
By comparison, \eqref{rate:optimal} is smaller in some extreme regimes, such as $d^{2/9}\epsilon^{-2/3}<\kappa<\epsilon^{2/3}d^{8/3}$ . We note that the comparison between RC-ULMC and RMM is not entirely fair since RMM uses a better discretization scheme than the Euler approximation \eqref{distributionofZ} used in RC-ULMC. It is possible  to include the RCD idea to RMM on ULMC as well, for a potentially better convergence rate. We leave that topic to future investigation.


\subsection{Tightness of the bound}\label{sec:tight}
Corollary~\ref{cor:rculmc} shows that the numerical cost is roughly $\widetilde{O}(d^{3/2}/\epsilon)$ when the problem is well conditioned. 
We show by use of an example that this bound is tight with respect to $d$ and $\epsilon$.
\begin{proposition}\label{prop:badexampleW22}
Let the target distribution be a standard Gaussian
\begin{equation} \label{eq:hs}
p_{\rx}(x)=\frac{1}{(2\pi)^{d/2}}\exp(-|x|^2/2)\,,
\end{equation}
which is the marginal distribution of the target distribution $p(x,v)=\frac{1}{(2\pi)^{d}}\exp(-|x|^2/2-|v|^2/2)$. Suppose the initial distribution $q_0$ is chosen to be 
\begin{equation} \label{eq:initial}
q_0(x,v)=\frac{1}{(2\pi)^{d}}\exp(-|x-{u}|^2/2-|v|^2/2)\,,
\end{equation}
with ${u}\in\mathbb{R}^d$ and ${u}_i=1/400$ for all $i$. Then if we choose $h$ so that $h<10^{-8}/d$, we have
\begin{equation}\label{eqn:badexampleW2bound2}
W_m\geq \frac{\exp\left(-4hm\right)}{800^2}\frac{d}{8}+\frac{d^{3/2}h}{320-464dh}\,,
\end{equation}
where $W_m := W_2(q_m,p)$, and $q_m$ is the probability distribution of $(x^m,v^m)$ generated by Algorithm~\ref{alg:RCD-OULMC} with $\gamma = 1/L=1$  using uniform coordinate sampling. Furthermore, to have $W_m\leq \epsilon$, one needs at least $M=\widetilde{O}(d^{3/2}/\epsilon)$ iterations.
\end{proposition}

The proof can be found in Supp.~\ref{sec:proofofprop:badexampleW22}. For this particular initialization, we have $W_0 = \sqrt{d}/400$. 
According to~\eqref{eqn:badexampleW2bound2}, we can guarantee $W_m\leq \epsilon$ only if both terms are smaller than $\epsilon$, meaning that (ignoring a $\log$ term)
\[
h\lesssim \frac{320\epsilon}{d^{3/2}}\,, \quad m\gtrsim \frac{1}{4h}\,
\]
which implies a cost of $\widetilde{O}(d^{3/2}/\epsilon)$. 

\section{Numerical experiments}\label{sec:numerics}

We give one example that demonstrates the improvement of RC-ULMC over the classical ULMC.

In the example, we repeat the Markov chain for $N$ independent trials and denote $\{x^{(i),M}\}_{i=1}^N$ the list of $N$ samples at $M$-th iteration. Since Wasserstein distance is difficult to measure directly numerically, especially when the underlying distribution function is presented by a list of particles, we evaluate the following error as a surrogate: 
\begin{equation}\label{MSEerror}
\textrm{Error}=\left\|\frac{1}{N}\sum^N_{i=1}\psi(x^{(i),M})-\EE_{p_\rx}(\psi)\right\|_2\,,
\end{equation}
where $\psi(x)$ is a matrix-valued function and referred to as the test function, $\|\cdot\|_2$ means the spectral norm of the matrix, and $\EE_{p_\rx}(\psi)$ is the expected value of $\psi$ with respect to the target distribution $p_\rx$.

In the example, we set the target distribution function to be
\[
p_{\rx}(x) \propto p_1(x)p_2(x)\,,
\]
with $p_2(x) = \exp\left(-\sum^{d}_{i=11}\frac{|x_i|^2}{2}\right)$ and
\[
p_1(x)=\exp\left(-\frac{1}{2}\mathsf{x}^\top\Gamma^\top\Gamma\mathsf{x}\right)\,,
\]
where $\mathsf{x}=\left(x_1,x_2,\dots,x_{10}\right)^\top$ is the list of first $10$ entries, and $\Gamma=\mathsf{T}+\frac{d}{10}I$. Here $I$ is the $10 \times 10$ identity matrix and $\mathsf{T}$ is a random matrix whose entries are i.i.d. standard Gaussian random variables.

In the simulation we set $d=100$, $N=10^5$, and let $\psi(x)=\mathsf{x}\mathsf{x}^\top\in\mathbb{R}^{10\times 10}$.


Initially, all particles are drawn from the density distribution $p_0(x,v)\propto p_1(\mathsf{x}-0.5\mathsf{e}_{10})p_2(x)\exp\left(-|v|^2/2\right)$, where $\mathsf{e}_{10}$ is a vector in $\mathbb{R}^{10}$ and all entries equal to $1$.
It is expected that the density of the target distribution is $p(x,v)\propto p_{\rx}(x)\exp\left(-|v|^2/2\right)$, making $p_{\rx}$ the marginal probability density. 

The result is plotted in Figure \ref{Figure3}. To run RC-ULMC, we use time stepsize $h=10^{-4}$. For comparison we also run ULMC, however, due to the cost difference per iteration, there is no standard choice of $h$ for ULMC for a fair comparison. Since $d=100$ in this example, per iteration, the cost of ULMC is $100$ times of that of RC-ULMC, we first experiment ULMC with $h = 10^{-2}$. It is clear that RC-ULMC, presented by the purple line achieves a lower error than ULMC with the same amount of cost.

We then test ULMC with different choices of $h$, hoping to find its best performance. As one increases $h$, the decay rate of error with respect to the cost increases too, but the error plateau is also higher, as one can see by comparing the yellow, red and blue lines in Figure~\ref{Figure3}, all produced by ULMC with different values of $h$. None of them, however, can compete with RC-ULMC regarding the level of error at the same cost.



\begin{figure}[htbp]
     \centering
      \includegraphics[height = 0.2\textheight, width = 0.6\textwidth]{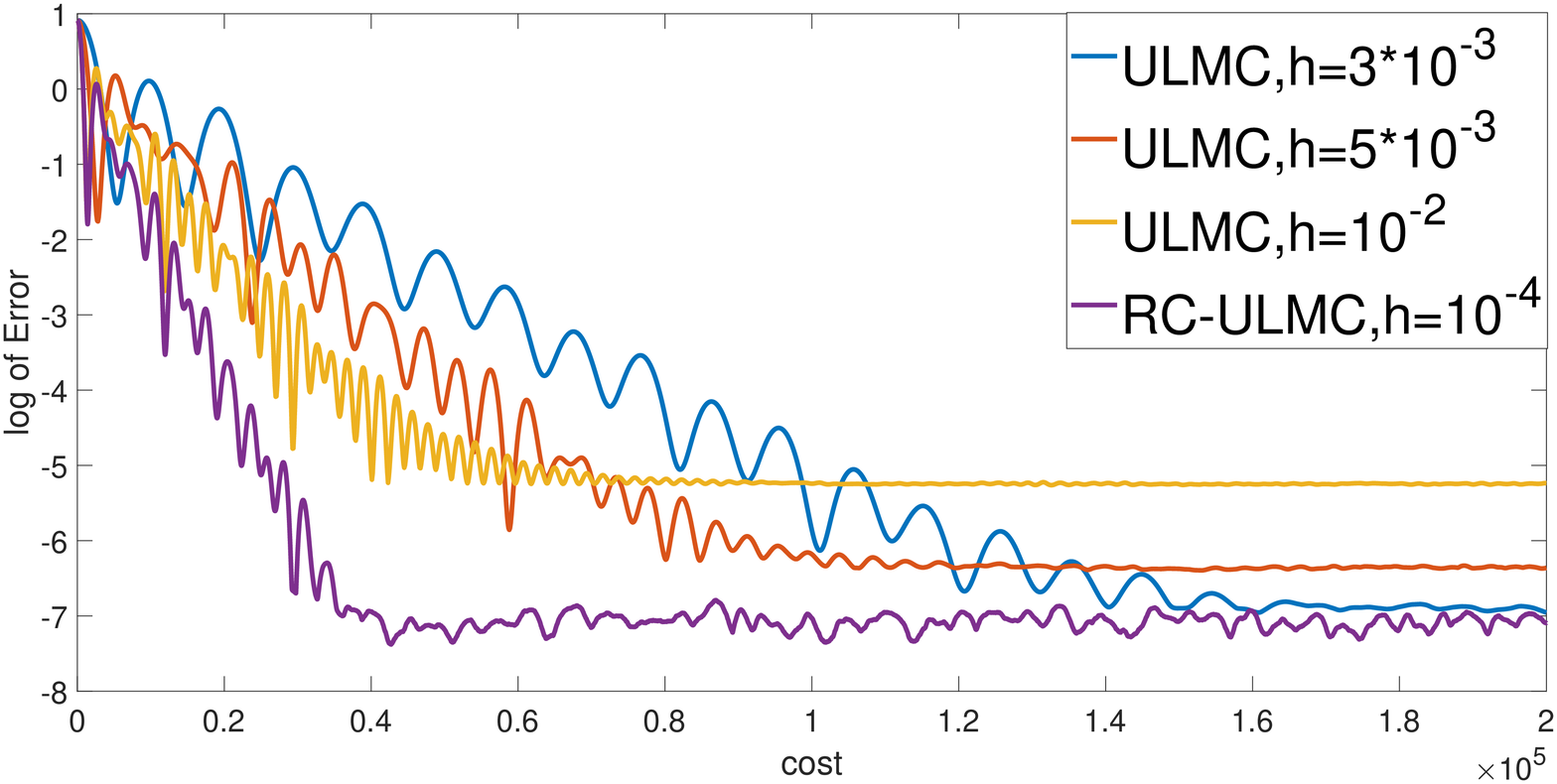}
     \caption{The decay of error with respect to the cost (the number $\partial f$ calculations). }
     \label{Figure3}
\end{figure}

\bibliographystyle{apalike}
\bibliography{enkf}

\appendix
\section{Proof of Theorem \ref{thm:convergenceulmc}}\label{sec:proofofthm:convergenceulmc1}
Recall 
\[
\left\{\begin{aligned}
&\rd X_t = V_t\rd t\,,\\
&\rd V_t = -2 V_t\rd t-\gamma\nabla f(X_t)\rd t+\sqrt{4 \gamma}\rd B_t\,.
\end{aligned}\right.
\]
Fixed $\gamma>0$, we define 
\[
    Y(t)=X\left(\sqrt{\gamma}t\right),\ Z(t)=\sqrt{\frac{1}{\gamma}}V\left(\sqrt{\gamma}t\right)\,,
\]
then we have
\begin{equation}\label{eqn:underdampedLangevin2}
   \left\{\begin{aligned}
&\rd Y_t = Z_t\rd t\,,\\
&\rd Z_t = -\gamma_1Z_t\rd t-\nabla f(Y_t)\rd t+\sqrt{2\gamma_1}\rd B_t\,,
\end{aligned}\right.
\end{equation}
where $\gamma_1=\frac{2}{\sqrt{\gamma}}$. This implies \eqref{eqn:underdampedLangevin} is equivalent to \eqref{eqn:underdampedLangevin2} with $\gamma_1=\frac{2}{\sqrt{\gamma}}$. 
Theorem~\ref{thm:convergenceulmc} is then a direct translation of  \cite[Theorem 2]{dalalyan2018sampling}.

\section{Proof of Theorem \ref{thm:contiummRCULMC}}\label{sec:proofofthm:convergenceulmc}

First, we introduce some notations. Denote the transition kernel by $\Xi$, meaning 
\[
(X^{m+1},V^{m+1})\stackrel{d}{=} \Xi((X^m,V^m), \cdot)\,.
\]
Moreover, we denote $\Xi^n$ the $n$-step transition kernel. The following proposition establishes the exponential convergence of the Markov chain. 
\begin{proposition}\label{prop:crcolmc}
Assume $f$ satisfies assumption \ref{assum:Cov} and
\[
h\leq \frac{\gamma \mu \min\{\phi_i\}}{312+12\gamma+8L+532L^2}\,,
\]
then there are constants $R_1>0,r_1>1$, such that for any $(x^0,v^0)\in\mathbb{R}^{2d}$
\begin{equation}\label{eqn:converge2}
\sup_{A\in\mathcal{B}(\mathbb{R}^{2d})}\left|\Xi^{md}((x^0,v^0),A)-\int_Ap(x,v)\rd x\rd v\right|\leq \left(|x^0-x^*|^2+|x^0-x^*+v^0|^2+1\right)R_1r^{-m}_1\,,
\end{equation}
where $x^*$ is the minimal point of $f(x)$ and $\mathcal{B}(\mathbb{R}^{2d})$ is the Borel set in $\mathbb{R}^{2d}$
\end{proposition}

We postpone the proof of the proposition to Section~\ref{proofofeqnconverge2} and first use the proposition to prove Theorem~\ref{thm:contiummRCULMC}:

\begin{proof}[Proof of Theorem \ref{thm:contiummRCULMC}]
First, if the distribution of $(X^m,V^m)$ is induced by $p$, then the marginal distribution of $(X_{r^m}(t),V_{r^m}(t))$ is preserved in \eqref{eqn:ULDSDE2continummx},\eqref{eqn:ULDSDE2continummv}. Since $X^{m+1}_i=X^m_i,V^{m+1}_i=V^m_i$ for $i\neq r^m$, we obtain $(X^{m+1},V^{m+1})\sim p$, meaning $p$ is the invariant measure of this Markov chain, concluding the first bullet point.

Since $q_0$ has finite second moment, multiply $q_0$ on both sides of \eqref{eqn:converge2} and integrate, we obtain
\[
\int_{\mathbb{R}^{2d}} |q_{md}(x,v)-p(x,v)|\rd x\rd v\leq C_0r^{-m}_1\,,
\]
where $C_0$ is a constant independent of $m$.

According to It\^o's formula and \eqref{eqn:ULDSDE2continummx}-\eqref{eqn:ULDSDE2continummxvfixed}, the boundedness of the second moment is preserved in each iteration. In particular:
\[
\EE\left(|X^{m+1}_{r^m}|^2+|V^{m+1}_{r^m}|^2|(X^m_{r^m},V^m_{r^m},r^m)\right)\leq c_1\EE\left(|X^{m}_{r^m}|^2+|V^{m}_{r^m}|^2|\right)+c_2\,,
\]
where $c_{1,2}$ are constants depending on $\gamma,L,x^*$. Since $q_0$ has finite second moments, $q_i$ has second moments. Therefore, multiply $q_i$ on both sides of \eqref{eqn:converge2} and integrate, we further have for $0\leq i\leq d-1$
\[
\int_{\mathbb{R}^{2d}} |q_{md+i}(x,v)-p(x,v)|\rd x\rd v\leq C_ir^{-m}_1\,,
\]
where $C_i$ is a constant independent of $m$. This proves \eqref{eqn:converge3} if we choose $R=(\max_i C_i)r_1$ and $r=r^{1/d}_1$.
\end{proof}

\subsection{Proof of Proposition~\ref{prop:crcolmc}}\label{proofofeqnconverge2}
To prove the proposition, we rely on the following result from \citep{MATTINGLY2002185}.
\begin{theorem}\label{thm:Mat}[\citet[Theorem 2.5]{MATTINGLY2002185}] 
Let $\{X^n\}^{\infty}_{n=0}$ denote the Markov chain on $\mathbb{R}^{2d}$ with transition kernel $\Xi$ and filtration $\mathcal{F}^n$. Let $\{X^n\}^{\infty}_{n=0}$ satisfy the following two conditions:
\begin{enumerate}[wide,labelwidth=0pt,labelindent=0pt,topsep=0pt,itemsep=1pt]
\item[Lynapunov condition:] There is a function $L:\mathbb{R}^{2d}\rightarrow [1,\infty)$, with $lim_{x\rightarrow\infty} L(x)=\infty$, and real number $\alpha\in(0,1)$, and $\beta\in [0,\infty)$ such that
\[
\EE\left(L(X^{n+1})\;\middle|\;\mathcal{F}^n\right)\leq \alpha L(X^{n})+\beta\,.
\]
\item[Minorization condition:] Define set $C\subset\mathbb{R}^{2d}$ 
\[
C=\left\{x\;\middle|\;x\in\mathbb{R}^d,\,L(x)\leq \frac{2\beta}{\gamma-\alpha}\right\}\,,
\]
for some $\gamma\in(\alpha^{1/2},1)$ and $L(x)$ comes from Lynapunov condition, there exists an $\eta>0$, and a probability measure $\mu$, with $\mu(C)=1$, such that for all $A\in\mathcal{B}(\mathbb{R}^{2d})$ and $x\in C$
\[
\Xi(x,A)\geq \eta \mu(A)\,.
\]
\end{enumerate}
Then the Markov chain $\{X^n\}^\infty_{n=0}$ has a unique invariant measure $\pi$. Furthermore, there is a constant $r\in(0,1)$ and $R\in(0,\infty)$ such that, for any $x_0\in\mathbb{R}^{2d}$:
\begin{equation}\label{eqn:mat}
\sup_{A\in\mathcal{B}(\mathbb{R}^{2d})}\left|\Xi^{m}(x^0,A)-\pi(A)\right|\leq L(x_0)Rr^{-m}\,.
\end{equation}
\end{theorem}

Comparing this theorem and the Proposition~\ref{prop:crcolmc}, we essentially need to prove the $d$-step chain $\left\{(X^{md},V^{md})\right\}^\infty_{m=0}$ satisfies both conditions. To do so, we first claim two results.
\begin{lemma}[Lyapunov condition]\label{lem:convergenceulmclemma} Assume $f$ satisfies Assumption \ref{assum:Cov} and 
\[
\gamma\leq\frac{1}{L},\quad h\leq \frac{\gamma \mu \min\{\phi_i\}}{312+12\gamma+8L+432L^2}\,.
\]
Let
\[
L(x,v)=|x-x^*|^2+|x-x^*+v|^2+1\,,
\]
then we have
\begin{equation}\label{eqn:lyconditionlemma}
\EE\left(L(X^{m+1},V^{m+1})\;\middle|\;\mathcal{F}^m\right)\leq \alpha L(X^{m},V^{m})+\beta
\end{equation}
with 
\[
\alpha=1-\gamma \mu h/2,\qquad \beta=(3780\gamma +\gamma \mu/2)h.
\]
\end{lemma}
\begin{lemma}[Minorization condition]\label{lemma:converge2}
Under conditions of Lemma \ref{lem:convergenceulmclemma}, define
\[
C=\left\{(x,v)\;\middle|\;L(x,v)\leq \frac{2\beta}{\gamma-\alpha}\right\}\,,
\]
then there exists an $\eta>0$, and a probability measure $\mathcal{M}$, with $\mathcal{M}(C)=1$, such that
\begin{equation}\label{minorizationcondition}
\Xi^d(x,A)\geq \eta \mathcal{M}(A),\quad \forall A\in\mathcal{B}(\mathbb{R}^{2d}),x\in C\,.
\end{equation}
\end{lemma}

Now, we are ready to prove Proposition \ref{prop:crcolmc}.
\begin{proof}[Proof of Proposition \ref{prop:crcolmc}]
Define $(Y^m,Z^m) = (X^{md},V^{md})$, then $\left\{\left(Y^{m},Z^{m}\right)\right\}^\infty_{m=0}$ is a Markov chain with transition kernel $\widetilde{\Xi}=\Xi^d$ and filtration $\widetilde{\mathcal{F}}^m=\mathcal{F}^{md}$. We will prove $\left\{\left(Y^{m},Z^{m}\right)\right\}^\infty_{m=0}$ satisfies the conditions in Theorem~\ref{thm:Mat} with $L(x,v) =|x-x^*|^2+|x-x^*+v|^2+1$.

We now show $\left\{(Y^m,Z^m)\right\}^\infty_{m=0}$ satisfies conditions in Theorem~\ref{thm:Mat} with $L(x,v) =|x-x^*|^2+|x-x^*+v|^2+1$, $\alpha = \alpha_1^d$ and $\beta=d\beta_1$, and $\pi$ is induced by $p$. Indeed, we use Lemma \ref{lem:convergenceulmclemma} \eqref{eqn:lyconditionlemma} iteratively for $d$ times:
\[
\EE\left(L\left(Y^{m+1},Z^{m+1}\right)\;\middle|\;\widetilde{\mathcal{F}}^{m}\right)\leq \alpha_1^d L\left(Y^{m},Z^{m}\right)+d\beta_1\,,
\]
which implies $\left\{\left(Y^{m},Z^{m}\right)\right\}^\infty_{m=0}$ satisfies Lyapunov condition in Theorem \ref{thm:Mat} with $\alpha = \alpha_1^d$ and $\beta=d\beta_1$. Next, since $\widetilde{\Xi}=\Xi^d$, Lemma \ref{lemma:converge2} directly implies the minorization condition. Therefore, by Theorem~\ref{thm:Mat}:
\begin{equation*}
\sup_{A\in\mathcal{B}(\mathbb{R}^{2d})}\left|\widetilde{\Xi}^{m}((x^0,v^0),A)-\pi(A)\right|\leq L(x^0,v^0)Rr^{-m}\,,
\end{equation*}
concluding the proposition by substituting $\widetilde{\Xi}=\Xi^d$ and $\pi(A) = \int_Ap(x,v)\rd{x}\rd{v}$.
\end{proof}



\begin{proof}[Proof of Lemma \ref{lem:convergenceulmclemma}]
For simplicity, we omit $\mathcal{F}^n$ and assume $x^*=\vec{0}$. Define 
\[
w=x+v,\quad W(T^m)=X(T^m)+V(T^m),\quad W^m=X^m+V^m\,,
\]
then we can write $L(x,v)$ as
\[
L(x,w)=|x|^2+|w|^2+1\,.
\]
We will prove \eqref{eqn:lyconditionlemma} in this case. 

We first seperate $\EE\left(L(X^{m+1},W^{m+1})\right)$
\[
\EE\left(L(X^{m+1},W^{m+1})\right)=\sum^d_{i=1}\phi_i\EE\left(L(X^{m+1},W^{m+1})\;\middle|\;r^m=i\right)
\]
According to~\eqref{eqn:ULDSDE2continummx},\eqref{eqn:ULDSDE2continummv}, under condition $r^m=i$, we have
\begin{equation}\label{eqn:equationsxsw}
\left\{\begin{aligned}
X^{m+1}_{i}-X^m_{i}=&\frac{1-e^{-2h_{i}}}{2}W^m_{i}-\frac{1-e^{-2h_{i}}}{2}X^m_{i}+\frac{\gamma}{2}\int^{T^m+h_{i}}_{T^m}\left(1-e^{-2(T^m+h_{i}-s)}\right)\partial_{i} f(X(s))\rd s\\
&+\sqrt{\gamma}\int^{T^m+h_{i}}_{T^m}\left(1-e^{-2(T^m+h_{i}-s)}\right)\rd B_s\\
W^{m+1}_{i}-W^m_{i}=&\frac{1-e^{-2h_{i}}}{2}X^m_{i}-\frac{1-e^{-2h_{i}}}{2}W^m_{i}-\frac{\gamma}{2}\int^{T^m+h_{i}}_{T^m}\left(1+e^{-2(T^m+h_{i}-s)}\right)\partial_{i}f(X(s)) \rd s\\
&+\sqrt{4\gamma}\int^{T^m+h_{i}}_{T^m}\left(1+e^{-2(T^m+h_{i}-s)}\right)\rd B_s \\
\end{aligned}\right.
\end{equation}
Since
\[
\begin{aligned}
\EE\left(L(X^{m+1},W^{m+1})\right)=&L(X^m,W^m)+2\left\langle X^m,X^{m+1}-X^m\right\rangle+\left\langle W^m,W^{m+1}-W^m\right\rangle\\
&+\left|X^{m+1}-X^m\right|^2+\left|W^{m+1}-W^m\right|^2\,.
\end{aligned}
\]
we have
\begin{equation}\label{eqn:taylor}
\begin{aligned}
\EE\left(L(X^{m+1},W^{m+1})\;\middle|\;r^m=i\right)=&L(X^m,W^m)+2\EE\left[X^m_i\left(X^{m+1}_{i}-X^m_{i}\right)+W^m_i\left(W^{m+1}_{i}-W^m_{i}\right)\;\middle|\;r^m=i\right]\\
&+\EE\left[\left(X^{m+1}_{i}-X^m_{i}\right)^2+\left(W^{m+1}_{i}-W^m_{i}\right)^2\;\middle|\;r^m=i\right]\,.
\end{aligned}
\end{equation}

We deal with second term and third term in \eqref{eqn:taylor} separately. First, use \eqref{eqn:equationsxsw}, we can write
\begin{align*}
    & \EE\left[2X^m_i\left(X^{m+1}_{i}-X^m_{i}\right)\;\middle|\;r^m =i\right]\\
    & =2X^m_i\left[\frac{1-e^{-2h_{i}}}{2}W^m_{i}-\frac{1-e^{-2h_{i}}}{2}X^m_{i}\right]+\gamma\left(h_{i}-\frac{1-e^{-2h_i}}{2}\right)X^m_i\partial_{i} f(X^m_i)\\
& \quad +\gamma\EE\left[\int^{T^m+h_{i}}_{T^m}\left(1-e^{-2(T^m+h_{i}-s)}\right)X^m_i\left(\partial_{i} f(X(s))-\partial_{i} f(X^m)\right)\rd s\;\middle|\;r^m=i\right]
\end{align*}
and
\begin{align*}
& \EE\left[2W^m_i\left(W^{m+1}_{i}-W^m_{i}\right)\;\middle|\;r^m=i\right]\\
& =2W^m_i\left[\frac{1-e^{-2h_{i}}}{2}X^m_{i}-\frac{1-e^{-2h_{i}}}{2}W^m_{i}\right]-\gamma\left(h_{i}+\frac{1-e^{-2h_i}}{2}\right)W^m_i\partial_{i} f(X^m_i)\\
& \quad -\gamma\EE\left[\int^{T^m+h_{i}}_{T^m}\left(1+e^{-2(T^m+h_{i}-s)}\right)W^m_i\left(\partial_{i} f(X(s))-\partial_{i} f(X^m)\right)\rd s\;\middle|\;r^m=i\right]
\end{align*}
Since $h_i\leq \frac{1}{20}$, we have
\[
|1-e^{-2h_i}|<2h_i+2h^2_i\,,
\]
which implies
\begin{align} 
\nonumber
& \EE\left[2X^m_i\left(X^{m+1}_{i}-X^m_{i}\right)\mid r^m=i\right] \\
\nonumber
& \leq 2h_iX^m_i\left[W^m_{i}-X^m_{i}\right] +4h^2_i\left[|W^m_{i}|^2+2|X^m_{i}|^2\right]+2\gamma h^2_i \left(|X^m_i|^2+|\partial_{i} f(X^m_i)|^2\right)\\
\label{eqn:Lxi}
& \quad +2\gamma h^2_i|X^m_i|^2+2\EE\left(\sup_{T^m\leq t\leq T^m+h_i}\left|\partial_{i} f(X(t))-\partial_{i} f(X^m)\right|^2\;\middle|\;r^m=i\right)
\end{align}
and
\begin{align}
\nonumber
& \EE\left[2W^m_i\left(W^{m+1}_{i}-W^m_{i}\right)\mid r^m=i\right] \\
\nonumber
& \leq 2h_iW^m_i\left[X^m_{i}-W^m_{i}\right]-2\gamma h_i W^m_i\partial_{i} f(X^m_i) +4h^2_i\left[2|W^m_{i}|^2+|X^m_{i}|^2\right]+2\gamma h^2_i \left(|W^m_i|^2+|\partial_{i} f(X^m_i)|^2\right)\\
\label{eqn:Lwi}
& \quad +4\gamma h^2_i|W^m_i|^2+2\EE\left(\sup_{T^m\leq t\leq T^m+h_i}\left|\partial_{i} f(X(t))-\partial_{i} f(X^m)\right|^2\;\middle|\;r^m=i\right)
\end{align}
From Lemma \ref{lemma:thmconvergence} \eqref{equ:boundofsupx}, we have
\begin{align*}
& \EE\left(\sup_{T^m\leq t\leq T^m+h_i}\left|\partial_{i} f(X(t))-\partial_{i} f(X^m_i)\right|^2\;\middle|\;r^m=i\right) \\
&\leq L^2_i \EE\left(\sup_{T^m\leq t\leq T^m+h_i}\left|X_i(t)-X^m_i\right|^2\;\middle|\;r^m=i\right)\\
&\leq L^2_i\left[54h^2_i\left(|X^m_i|^2+|W^m_i|^2\right)+9\gamma^2h^2_i|\partial_if(X^m)|^2+540\gamma h^3_i\right]\,.
\end{align*}
By substituting into \eqref{eqn:Lxi} and \eqref{eqn:Lwi}, we find that the first term can be bounded  as follows:
\begin{align}
\nonumber
&\EE\left[2X^m_i\left(X^{m+1}_{i}-X^m_{i}\right)\;\middle|\;r^m=i\right]+\EE\left[2W^m_i\left(W^{m+1}_{i}-W^m_{i}\right)\;\middle|\;r^m=i\right]\\
\nonumber
& \leq 2h_iX^m_i\left[W^m_{i}-X^m_{i}\right]+2h_iW^m_i\left[X^m_{i}-W^m_{i}\right]-2\gamma h_i W^m_i\partial_{i} f(X^m)\\
\nonumber
&\quad +h^2_i\left(12+6\gamma+216L^2_i\right)\left[|W^m_{i}|^2+|X^m_{i}|^2\right]+h^2_i\left(4\gamma+18\gamma^2\right)|\partial_{i} f(X^m)|^2 +2160\gamma h^3_i\\
\nonumber
& \leq 2h_iX^m_i\left[W^m_{i}-X^m_{i}\right]+2h_iW^m_i\left[X^m_{i}-W^m_{i}\right]-2\gamma h_i W^m_i\partial_{i} f(X^m)\\
\label{eqn:taylor1}
& \quad +h^2_i\left(30+6\gamma+4L_i+216L^2_i\right)\left[|W^m_{i}|^2+|X^m_{i}|^2\right]+2160\gamma h^3_i\,,
\end{align}
where in the last inequality we use
\[
|\partial_{i} f(X^m)|^2\leq L^2_i|X^m_i|^2 
\]
since $x^*=\vec{0}$ and $\|\partial_{ii}f\|_{\infty}\leq L_i$.

Next, we deal with third term in \eqref{eqn:taylor}, notice Lemma \ref{lemma:thmconvergence} \eqref{equ:boundofsupx}, \eqref{equ:boundofsupw}, we have
\begin{align}
\nonumber
&\EE\left[\left(X^{m+1}_{i}-X^m_{i}\right)^2+\left(W^{m+1}_{i}-W^m_{i}\right)^2\;\middle|\;r^m=i\right]\\
\nonumber
& \leq  108h^2_i\left(|X^m_i|^2+|W^m_i|^2\right)+18\gamma^2h^2_i|\partial_if(X^m)|^2+1620\gamma h_i\\
& \leq  126h^2_i\left(|X^m_i|^2+|W^m_i|^2\right)+1620\gamma h_i\,, \label{eqn:taylor2}
\end{align}
where we use $|\partial_{i} f(X^m)|^2\leq L^2_i|X^m_i|^2$ again.

By substituting \eqref{eqn:taylor1},\eqref{eqn:taylor2} into \eqref{eqn:taylor}, we obtain
\begin{align*}
& \EE\left(L(X^{m+1},W^{m+1})\;\middle|\;r^m=i\right) \\
& \leq L(X^m,W^m)+2h_iX^m_i\left[W^m_{i}-X^m_{i}\right]+2h_iW^m_i\left[X^m_{i}-W^m_{i}\right]-2\gamma h_i W^m_i\partial_{i} f(X^m)\\
& \quad +h^2_i\left(156+6\gamma+4L_i+216L^2_i\right)(L(X^m,W^m)-1)+3780\gamma h_i\,,
\end{align*}
which implies that 
\begin{align*}
& \EE\left(L(X^{m+1},W^{m+1})\right) \\
& \leq L(X^m,W^m)+2h\left\langle X^m,W^m-X^m\right\rangle +2h\left\langle W^m,X^m-W^m\right\rangle-2\gamma h \left\langle W^m,\nabla f(X^m)\right\rangle\\
& \quad +\frac{\left(156+6\gamma+4L+216L^2\right)h^2}{\min\{\phi_i\}}(L(X^m,W^m)-1)+3780\gamma h\,.
\end{align*}
Similar to \eqref{eqn:Iab}-\eqref{Iinequality}, we have
\[
2h\left\langle X^m,W^m-X^m\right\rangle +2h\left\langle W^m,X^m-W^m\right\rangle-2\gamma h \left\langle W^m,\nabla f(X^m)\right\rangle\leq -\gamma \mu h \left(|X^m|^2+|W^m|^2\right)\,,
\]
which implies
\[
L(X^m,W^m)-1+2h\left\langle X^m,W^m-X^m\right\rangle +2h\left\langle W^m,X^m-W^m\right\rangle-2\gamma h \left\langle W^m,\nabla f(X^m)\right\rangle\leq (1-\gamma\mu h)(L(X^m,W^m)-1)\,.
\]
Since $h\leq \frac{\gamma \mu \min\{\phi_i\}}{312+12\gamma+8L+432L^2}$, we finally prove \eqref{eqn:lyconditionlemma}.
\end{proof}
\begin{proof}[Proof of Lemma \ref{lemma:converge2}]
To prove \eqref{minorizationcondition}, we define another Markov chain, for fixed $\left(\wsx^m,\wsv^m\right)$, the $\left(\wsx^{m+1},\wsv^{m+1}\right)$ is produced by the following coupled SDEs: Define
\[
\widetilde{T}^n=\sum^{n}_{i=1}h_i,\quad \widetilde{T}^0=0,
\]
then for $\widetilde{T}^n\leq t\leq \widetilde{T}^{n+1}$ and $0\leq n\leq d-1$
\[
\left\{\begin{aligned}
\wsv_{i}(t)=&\wsv_{i}\left(\widetilde{T}^n\right)e^{-2(t-\widetilde{T}^n)}-\gamma\int^{t}_{\widetilde{T}^n}e^{-2(t-s)}\partial_{i}f\left(\wsx(s)\right) \rd s+\sqrt{4\gamma}\int^{t}_{\widetilde{T}^n}e^{-2(t-s)}\rd B_s \\
\wsx_{i}(t)=&\wsx_{i}\left(\widetilde{T}^n\right)+\frac{1-e^{-2t}}{2}\wsv_{i}\left(\widetilde{T}^n\right)-\frac{\gamma}{2}\int^t_{\widetilde{T}^n}\left(1-e^{-2(t-s)}\right)\partial_{i} f\left(\wsx(s)\right)\rd s\\
&+\sqrt{\gamma}\int^t_{\widetilde{T}^n}\left(1-e^{-2(t-s)}\right)\rd B_s
\end{aligned}\right.
\]
and $\wsx_{i}(t)=\wsx_i\left(\widetilde{T}^n\right),\wsv_{i}(t)=\wsv_i\left(\widetilde{T}^n\right)$ for $i\neq n$ with initial condition $\wsx(0)=\wsx^n$, $\wsv(0)=\wsv^n$. We set $\left(\wsx^{n+1},\wsv^{n+1}\right)=\left(\wsx\left(\widetilde{T}^d\right),\wsv\left(\widetilde{T}^d\right)\right)$. We denote the transition kernel by $\Xi_2$, then we have the following properties: 
\begin{itemize}
\item For any $x\in C$ and $A\in\mathcal{B}(\mathbb{R}^{2d})$, we have
\[
\Xi^d(x,A)\geq \Pi^{d}_{i=1}\phi_i \Xi_2(x,A)>0\,.
\]
\item $\Xi_2$ has a continuous postive density.
\end{itemize}
Since the new transition kernel $\Xi_2$ has a continuous postive density, according to Lemma 2.3 in \cite{MATTINGLY2002185}, there exists an $\eta>0$, and a probability measure $\mathcal{M}$, with $\mathcal{M}(C)=1$, such that
\[
\Xi_2(x,A)>\eta \mathcal{M}(A),\quad \forall A\in\mathcal{B}(\mathbb{R}^{2d}),x\in C\,,
\]
which implies
\[
\Xi^d(x,A)\geq \Pi^{d}_{i=1}\phi_i  \Xi_2(x,A)>\Pi^{d}_{i=1}\phi_i  \eta \mathcal{M}(A),\quad \forall A\in\mathcal{B}(\mathbb{R}^{2d}),x\in C\,.
\]
This proves \eqref{minorizationcondition}.
\end{proof}
\section{Proof of Theorem \ref{thm:rculmc}}\label{sec:proofofthm:rculmc}
Recall
\[
T^m=\sum^{m-1}_{n=0} h_{r^n}\,.
\]

According to algorithm \ref{alg:RCD-OULMC}, RC-ULMC can be seen as drawing $(x^0,v^0)$ from distribution induced by $q_0$, and
update $(x^m,v^m)$ using the following coupled SDEs for $T^m<h\leq T^{m+1}$:
\begin{equation}\label{eqn:ULDSDE2SAGA}
\left\{\begin{aligned}
&\mathrm{V}_{r^m}(t)=v^m_{r^m}e^{-2(t-T^m)}-\gamma\partial_{r^m}f(x^m) \int^{t}_{T^m}e^{-2(t-s)}\rd s+\sqrt{4\gamma}e^{-2 (t-T^m)}\int^{t}_{T^m}e^{2 s}\rd B_s \,,\\
&\mathrm{X}_{r^m}(t)=x^m_{r^m}+\int^{t}_{T^m} \mathrm{V}_{r^m}(s)\rd s\,,
\end{aligned}\right.\,
\end{equation}
and $\mathrm{X}_{i}(t)=x^m_i,\mathrm{V}_{i}(t)=v^m_i$ for $i\neq r^m$, where $B_s$ is a one dimensional Brownian motion. And we let $(x^{m+1},v^{m+1})=(\mathrm{X}(T^{m+1}),\mathrm{V}(T^{m+1}))$. 

Define another trajectory of sampling by setting  $(\widetilde{x}^0,\wv^0)$ to be drawn from distribution induced by $p$ and generating $\left(\wrx(t),\wrv(t)\right)$ according to \eqref{eqn:ULDSDE2continummx}-\eqref{eqn:ULDSDE2continummxvfixed} with $\left(\wrx(0),\wrv(0)\right)=(\widetilde{x}^0,\wv^0)$. Denote 
\begin{equation}\label{eqn:wxwvm}
    \wx^m=\wrx(T^m),\ \wv^m=\wrv(T^m),
\end{equation}
it was proved in Theorem \ref{thm:contiummRCULMC} that $(\wx^m,\wv^m)$ can be seen as drawn from distribution induced by $p$ for all $m\geq 0$.

Now, we define $w^m=x^m+v^m$ and $\ww^m=\wx^m+\wv^m$, and denote $u_m(x,w)$ the probability density of $(x^m,w^m)$ and $u^\ast(x,w)$ the probability density of $(x,w)$ if $(x,v=w-x)$ is distributed according to density function $p$. From~\cite{Cheng2017UnderdampedLM}, we have:
\begin{equation}\label{trivialinequlaitySAGA}
 |x^m-x|^2+|v^m-v|^2\leq 4(|x^m-x|^2+|w^m-w|^2)\leq 16(|x^m-x|^2+|v^m-v|^2)\,\\
 \end{equation}
 and
 \begin{equation}\label{trivialinequlaitySAGA2}
  W^2_2(q_{m},p)\leq  4W^2_2(u_{m},u^*)\leq  16W^2_2(q_{m},p)\,.
\end{equation}
Therefore, quantifying the convergence from $q_m$ to $p$ is the same as showing the convergence from $u_m$ to $u^\ast$.

We then also define $\Delta^m$ 
\begin{equation}\label{eqn:deltaulmc}
\Delta^m = \sqrt{|\wx^m-x^m|^2+|\ww^m-w^m|^2}\,
\end{equation}
and pick $\left(\wx^0,\wv^0\right)$ such that
\[
W^2_2(u_0,u^*)=\EE|\Delta^0|^2\,.
\]
Since $(\wx^m,\wv^m)\sim p$, we have $W^2_2(u_m,u^\ast)\leq\EE|\Delta^m|^2$ and we only need to bound $\EE|\Delta^m|^2$. 

Now, we give the following iteration formula for $\EE|\Delta^m|^2$:
\begin{proposition}\label{prop:rculmc} Under conditions of Theorem \ref{thm:rculmc}, assume $\{(x^m,v^m)\}$ is defined in \eqref{eqn:ULDSDE2SAGA}, $\{(\wx^m,\wv^m)\}$ is defined in \eqref{eqn:wxwvm} and $\{\Delta^m\}$ comes from \eqref{eqn:deltaulmc},
\begin{equation}\label{prop:iterationulmc}
\begin{aligned}
\EE|\Delta^{m+1}|^2\leq \left(1-\frac{\gamma\mu h}{4}\right)\EE|\Delta^m|^2+\frac{6\gamma^2h^3}{\mu}\sum^d_{i=1}\frac{L^2_i}{\phi^2_i}\,.
\end{aligned}
\end{equation}
\end{proposition}

We prove Propositon \ref{prop:rculmc} in Section \ref{sec:proofofproprculmc}. Now, we are ready to prove the Theorem \ref{thm:rculmc}:
\begin{proof}[Proof of Theorem \ref{thm:rculmc}]
Using \eqref{prop:iterationulmc} iteratively, we have
\[
\EE|\Delta^{m}|^2\leq \left(1-\frac{\mu\gamma h}{4}\right)^m\EE\left|\Delta^0\right|^2+\frac{24\gamma h^2}{\mu^2}\sum^d_{i=1}\frac{L^2_i}{\phi^2_i}
\leq \exp\left(-\frac{\mu\gamma mh}{4}\right)\EE\left|\Delta^0\right|^2+\frac{24\gamma h^2}{\mu^2}\sum^d_{i=1}\frac{L^2_i}{\phi^2_i}\,.
\]
Taking square root on both sides, and use \eqref{trivialinequlaitySAGA2}, we proves \eqref{thm:iterationulmc}.
\end{proof}
\subsection{Proof of Proposition \ref{prop:rculmc}}\label{sec:proofofproprculmc}
The Proposition \ref{prop:rculmc} is a direct result of the following lemma:

This lemma analyzes the terms in Proposition \ref{prop:rculmc} component-wisely.
\begin{lemma}\label{lem:rculmc}
Assume $f$ satisfies assumption \ref{assum:Cov}, if $\{(x^m,v^m)\}$ is defined in \eqref{eqn:ULDSDE2SAGA}, $\{(\wx^m,\wv^m)\}$ is defined in \eqref{eqn:wxwvm} and $\{\Delta^m\}$ comes from \eqref{eqn:deltaulmc},
then for any $m\geq 0$ and $i=1,2,\dots,d$:
\begin{equation}\label{eqn:resultoflemulmc}
\begin{aligned}
\EE|\Delta^{m+1}_i|^2\leq &\left(1+\frac{\gamma\mu h}{2}+\frac{20h^2}{\phi_i}\right)\EE|\Delta^m_i|^2+\frac{10\gamma^2h^2}{\phi_i}\EE\left(\left|\partial_i f(\wx^m)-\partial_i f(x^m)\right|^2\right)\\
&+2\phi_i\EE \mathrm{K}_i+\frac{6\gamma^2 L^2_ih^3}{\mu \phi^2_i}\,,
\end{aligned}
\end{equation}
where
\begin{equation}\label{eqn:KSAGA}
\mathrm{K}_i=\left(A^m_i-(\ww^m_i-w^m_i)\right)(\ww^m_i-w^m_i)+\left(C^m_i-(\wx^m_i-x^m_i)\right)(\wx^m_i-x^m_i)\,,
\end{equation}
and
\begin{align}\label{eqn:ASAGA}
A^m_i& =(h_i+e^{-2h_i})(\ww^m_i-w^m_i)+(1-h_i-e^{-2h_i})(\wx^m_i-x^m_i)
-\frac{\gamma\left(1-e^{-2h_i}\right)}{2}\left[\partial_i f(\wx^m)-\partial_i f(x^m)\right]\,, \\
\label{eqn:CSAGA}
C^m_i& =(1-h_i)(\wx^m_i-x^m_i)+h_i(\ww^m_i-w^m_i)\,,
\end{align}
\end{lemma}

To show Proposition~\ref{prop:rculmc} amounts to summing up all components in Lemma~\ref{lem:rculmc}. In particular, as will be shown in Lemma A.5.3, the third term in \eqref{eqn:resultoflemulmc} will contribute a negative $\EE|\Delta^m_i|^2$ term. If it dominates $\frac{\gamma\mu h}{2}$ in the coefficient of the first term, the decay of the error is expected.

\begin{proof}[Proof of Lemma \ref{lem:rculmc}]
In the $m$-th time step, we have
\[
\mathbb{P}(r^m=i)=\phi_i,\quad \mathbb{P}(r^m\neq i)=1-\phi_i\,.
\]
This implies
\begin{equation}\label{eqn:pickrulmc}
\begin{aligned}
\EE|\Delta^{m+1}_i|^2&=\phi_i\EE\left(\left.|\Delta^{m+1}_i|^2\right|r^m=i\right)+(1-\phi_i)\EE\left(\left.|\Delta^{m+1}_i|^2\right|r^m\neq i\right)\\&=\phi_i\EE\left(\left.|\Delta^{m+1}_i|^2\right|r^m=i\right)+(1-\phi_i)\EE\left|\Delta^{m}_i\right|^2\,.
\end{aligned}
\end{equation}
It suffices to bound the first term of \eqref{eqn:pickrulmc}. Under condition $r^m=i$, we first divide $|\Delta^{m+1}_i|^2$ into different parts under the condition $r^m=i$, and compare \eqref{eqn:ULDSDE2SAGA} and \eqref{eqn:ULDSDE2continummx}-\eqref{eqn:ULDSDE2continummxvfixed} for:
\[
\begin{aligned}
|\Delta^{m+1}_i|^2=&\left|(\wv^m_i-v^m_i)e^{-2h_i}+(\wx^m_i-x^m_i)+\int^{T^{m}+h_i}_{T^m}\left(\wrv_i(s)-\rv_i(s)\right)\rd s\right.\\
&\left.-\gamma\int^{T^{m}+h_i}_{T^m}e^{-2(T^{m}+h_i-s)}\left[\partial_i f\left(\wrx(s)\right)-\partial_i f(x^m)\right]\rd s\right|^2\\
&+\left|
(\wx^m_i-x^m_i)+\int^{T^{m}+h_i}_{T^m}\left(\wrv_i(s)-\rv_i(s)\right)\rd s
\right|^2\\
=&\left|\mathrm{I}^m_i\right|^2+\left|\mathrm{J}^m_i\right|^2\,,
\end{aligned}
\]
where we denote $\mathrm{I}^m_i$ and $\mathrm{J}^m_i$ the quantities in the first and second absolute value signs above respectively.

We try to bound $\EE\left(\left|\mathrm{I}^{m}_i\right|^2+\left|\mathrm{J}^m_i\right|^2\right)$ using $\EE|\Delta^m_i|^2$. We first try to seperate out ($x^m_i,\widetilde{x}^m_i,v^m_i,\widetilde{v}^m_i$) from $\mathrm{I}^m_i$ and $\mathrm{J}^m_i$. Denote
\begin{equation}\label{eqn:BSAGA}
\begin{aligned}
B^m_i=&\int^{T^{m}+h_i}_{T^m}\left(\wrv_i(s)-\rv_i(s)-(\wv^m_i-v^m_i)\right)\rd s\\
&-\gamma\int^{T^{m}+h_i}_{T^m}e^{-2(T^{m}+h_i-s)}\left[\partial_i f\left(\wrx(s)\right)-\partial_i f(\wx^m)\right]\rd s\,,
\end{aligned}
\end{equation}
and
\begin{equation}\label{eqn:DSAGA}
D^m_i=\int^{T^{m}+h_i}_{T^m}\left(\wrv_i(s)-\rv_i(s)-(\wv^m_i-v^m_i)\right)\rd s\,,
\end{equation}
according to the definition of $A^m_i$ and $C^m_i$ in~\eqref{eqn:ASAGA},\eqref{eqn:CSAGA}, we have:
\[
\mathrm{I}^{m}_i=A^m_i+B^m_i,\quad \mathrm{J}^m_i=C^m_i+D^m_i\,.
\]

Use Young's inequality, for any $a>0$, we have
\begin{equation}\label{IJ}
|\mathrm{I}^{m}_i|^2+|\mathrm{J}^{m}_i|^2\leq (1+a)\left(|A^{m}_i|^2+|C^{m}_i|^2\right)+\left(1+\frac{1}{a}\right)\left(|B^{m}_i|^2+|D^{m}_i|^2\right)\,.
\end{equation}

According to Lemma \ref{lem:BDrculmc} \eqref{bound:Brculmc},\eqref{bound:Drculmc}, we have
\begin{equation}\label{bound:Brculmc1}
\begin{aligned}
\EE\left(\left.|B^m_i|^2\right| r^m=i\right)\leq &\frac{16h^4_i}{3}\EE\left|\wv^m_i-v^m_i\right|^2+\frac{8\gamma^2 h^4_i}{3}\EE\left(\left|\partial_i f(\wx^m)-\partial_i f(x^m)\right|^2\right)+\frac{6\gamma^3L^2_i h^4_i}{5}\,,
\end{aligned}
\end{equation}
\begin{equation}\label{bound:Drculmc1}
\begin{aligned}
\EE\left(\left.|D^m_i|^2\right| r^m=i\right)\leq &\frac{8h^4_i}{3}\EE\left|\wv^m_i-v^m_i\right|^2+\frac{4\gamma^2 h^4_i}{3}\EE\left(\left|\partial_i f(\wx^m)-\partial_i f(x^m)\right|^2\right)+\frac{4\gamma^3 L^2_i h^6_i}{15}\,.
\end{aligned}
\end{equation}
Since $h_i<1/240$, we obtain
\begin{equation}\label{bound:BDrculmc}
\begin{aligned}
\EE\left(\left.|B^{m}_i|^2+|D^{m}_i|^2\right| r^m=i\right)\leq &8h^4_i\EE\left|\wv^m_i-v^m_i\right|^2+4\gamma^2 h^4_i\EE\left(\left|\partial_i f(\wx^m)-\partial_i f(x^m)\right|^2\right)\\
&+\frac{7\gamma^3L^2_i h^4_i}{5}\,.
\end{aligned}
\end{equation}
Substituting \eqref{bound:BDrculmc} into \eqref{IJ}, we have
\begin{equation}\label{Deltam+1firstpart}
\begin{aligned}
\EE\left(|\Delta^{m+1}_i|^2|r^m=i\right)\leq &(1+a)\EE\left(\left.|A^m_i|^2+|C^m_i|^2\right|r^m=i\right)\\
&+\left(1+\frac{1}{a}\right)8h^4_i\EE\left|\wv^m_i-v^m_i\right|^2+\left(1+\frac{1}{a}\right)4\gamma^2 h^4_i\EE\left(\left|\partial_i f(\wx^m)-\partial_i f(x^m)\right|^2\right)\\
&+\left(1+\frac{1}{a}\right)\frac{7\gamma^3L^2_i h^4_i}{5}\,.
\end{aligned}
\end{equation}

Considering
\begin{equation}\label{eqn:AiplucCi}
\begin{aligned}
\EE\left(\left.|A^m_i|^2+|C^m_i|^2\right|r^m=i\right)=&\EE|\Delta^m_i|^2+\EE\left(\left|A^m_i-(\ww^m_i-w^m_i)\right|^2+\left|C^m_i-(\wx^m_i-x^m_i)\right|^2\right)+2\EE\mathrm{K}_i\,,\\
\end{aligned}
\end{equation}
we need to give bounds to the second and the third terms. According to the definition of $A^m_i$ and $C^m_i$ in~\eqref{eqn:ASAGA},\eqref{eqn:CSAGA}, we have:
\begin{equation}\label{eqn:boundofAminusw}
\begin{aligned}
\EE\left(\left|A^m_i-(\ww^m_i-w^m_i)\right|^2\right)&\leq 4(1-h_i-e^{-2h_i})^2 \EE|\Delta^m_i|^2+2\gamma^2h^2_i\EE|\partial_i f(\wx^m)-\partial_i f(x^m)|^2\\
&\leq 4h^2_i \EE|\Delta^m_i|^2+2\gamma^2h^2_i\EE|\partial_i f(\wx^m)-\partial_i f(x^m)|^2\,,\\
\EE\left(\left|C^m_i-(\wx^m_i-x^m_i)\right|^2\right)&\leq 2h^2_i\EE|\Delta^m_i|^2\,,
\end{aligned}
\end{equation}
and thus:
\begin{equation}\label{eqn:AiplucCi2}
\begin{aligned}
\EE\left(\left.|A^m_i|^2+|C^m_i|^2\right|r^m=i\right)\leq& (1+6h^2_i)\EE|\Delta^m_i|^2+2\gamma^2h^2_i\EE|\partial_i f(\wx^m)-\partial_i f(x^m)|^2+2\EE\mathrm{K}_i\,.
\end{aligned}
\end{equation}

Substituting \eqref{eqn:AiplucCi2} into \eqref{Deltam+1firstpart} and using \eqref{eqn:pickrulmc}, we have
\begin{equation}\label{Deltam+1final1}
\begin{aligned}
\EE|\Delta^{m+1}_i|^2\leq &(1+a\phi_i+6h^2_i\phi_i(1+a))\EE|\Delta^m_i|^2+\left(1+\frac{1}{a}\right)\phi_i8h^4_i\EE\left|\wv^m_i-v^m_i\right|^2\\
&+\left[(1+a)\phi_i2\gamma^2h_i^2+\left(1+\frac{1}{a}\right)\phi_i4\gamma^2 h^4_i\right]\EE\left(\left|\partial_i f(\wx^m)-\partial_i f(x^m)\right|^2\right)\\
&+2(1+a)\phi_i\EE\mathrm{K}_i+\left(1+\frac{1}{a}\right)\phi_i\frac{7\gamma^3L^2_i h^4_i}{5}\,.
\end{aligned}
\end{equation}
To find a good choice of $a$, we cite the estimate in Lemma \ref{lem:ACrculmc} that states the fourth term of \eqref{Deltam+1final1}, when summed up in index $i$, will contribute $-\gamma \mu h(1+a)\EE|\Delta|^2$:
\begin{equation}\label{boundofKi}
    \sum^d_{i=1}\phi_i\EE\mathrm{K}_i\leq \left(-\frac{\gamma \mu h}{2}+\frac{3h^2}{\min\{\phi_i\}}\right)\EE|\Delta^m|^2\,.
\end{equation}
Therefore, the coefficient in the first line should not exceed the $1+\gamma\mu h$, which puts a good choice of $a$ to be
\[
a=\frac{\gamma \mu h_i}{2}=\frac{\gamma\mu h}{2\phi_i}<1\,,
\] 
which also implies
\[
1+\frac{1}{a}\leq \frac{4}{\gamma\mu h_i}\,.
\]
Substituting this into \eqref{Deltam+1final1} and using $\EE\left|\wv^m_i-v^m_i\right|\leq 2\EE|\Delta^m|^2$, $h_i\phi_i=h$, we have
\begin{equation}\label{Deltam+1final12}
\begin{aligned}
\EE|\Delta^{m+1}_i|^2\leq &\left(1+\frac{\gamma\mu h}{2}+\frac{12h^2}{\phi_i}+\frac{32h^3}{\gamma\mu \phi^2_i}\right)\EE|\Delta^m_i|^2\\
&+\left(\frac{4\gamma^2h^2}{\phi_i}+\frac{16\gamma h^3}{\mu \phi^2_i}\right)\EE\left(\left|\partial_i f(\wx^m)-\partial_i f(x^m)\right|^2\right)\\
&+2\left(\phi_i+\frac{\gamma\mu h}{2}\right)\EE\mathrm{K}_i+\frac{6\gamma^2 L^2_ih^3}{\mu \phi^2_i}\,.
\end{aligned}
\end{equation}
To further control the $\EE\mathrm{K}_i$ term, we note, using~\eqref{eqn:boundofAminusw} again:
\begin{equation}\label{eqn:boundofKifirst}
   \begin{aligned}
&\EE\left[\left(A^m_i-(\ww^m_i-w^m_i)\right)(\ww^m_i-w^m_i)\right]\\
& \qquad \qquad \leq \frac{1}{2}\left(\frac{\phi_i}{h}\EE\left|A^m_i-(\ww^m_i-w^m_i)\right|^2+\frac{h}{\phi_i}\EE\left|\ww^m_i-w^m_i\right|^2\right)\\
& \qquad \qquad \leq \frac{5h}{2\phi_i}\EE|\Delta^m_i|^2+\frac{\gamma^2h}{\phi_i}\EE|\partial_i f(\wx^m)-\partial_i f(x^m)|^2\,,
\end{aligned} 
\end{equation}
\begin{equation}\label{eqn:boundofKisecond}
\begin{aligned}
&\EE\left[\left(C^m_i-(\wx^m_i-x^m_i)\right)(\wx^m_i-x^m_i)\right]\\
& \qquad \qquad \leq \frac{1}{2}\left(\frac{\phi
_i}{h}\EE\left|C^m_i-(\wx^m_i-x^m_i)\right|^2+\frac{h}{\phi_i}\EE\left|\wx^m_i-x^m_i\right|^2\right)\\
& \qquad \qquad \leq \frac{3h}{2\phi_i}\EE|\Delta^m_i|^2\,.
\end{aligned}
\end{equation}


Therefore we finally have 
\begin{align*}
\EE|\Delta^{m+1}_i|^2\leq &\left(1+\frac{\gamma\mu h}{2}+\frac{16h^2}{\phi_i}+\frac{32h^3}{\gamma\mu \phi^2_i}\right)\EE|\Delta^m_i|^2\\
&+\left(\frac{5\gamma^2h^2}{\phi_i}+\frac{16\gamma h^3}{\mu \phi^2_i}\right)\EE\left(\left|\partial_i f(\wx^m)-\partial_i f(x^m)\right|^2\right)+2\phi_i\mathrm{K}_i+\frac{6\gamma^2 L^2_ih^3}{\mu \phi^2_i}\,.
\end{align*}
Use $h<\frac{\gamma\mu \min\{\phi_i\}}{240}$ and $h/\phi_i<1$, we have \eqref{eqn:resultoflemulmc}.
\end{proof}

Now, we are ready to prove Proposition \ref{prop:rculmc}.

\begin{proof}[Proof of Proposion \ref{prop:rculmc}]
Use Lemma \ref{lem:rculmc} and sum \eqref{eqn:resultoflemulmc} up, we obtain
\begin{equation}\label{proofofpropulmc1}
\begin{aligned}
\EE|\Delta^{m+1}|^2\leq &\left(1+\frac{\gamma\mu h}{2}+\frac{20h^2}{\min\{\phi_i\}}\right)\EE|\Delta^m|^2+\frac{10\gamma^2h^2}{\min\{\phi_i\}}\EE\left(\left|\nabla f(\wx^m)-\nabla f(x^m)\right|^2\right)\\
&+2\sum^d_{i=1} \phi_i\EE \mathrm{K}_i+\frac{6\gamma^2h^3}{\mu}\sum^d_{i=1}\frac{L^2_i}{\phi^2_i}\\
\leq &\left(1+\frac{\gamma\mu h}{2}+\frac{30h^2}{\min\{\phi_i\}}\right)\EE|\Delta^m|^2+2\sum^d_{i=1} \phi_i\EE \mathrm{K}_i+\frac{6\gamma^2h^3}{\mu}\sum^d_{i=1}\frac{L^2_i}{\phi^2_i}\,,
\end{aligned}
\end{equation}
where we use $f$ is $L$-Lipschitz in the second inequality and $\gamma L\leq 1$. Then, we use Lemma \ref{lem:ACrculmc} \eqref{ACboundSAGA} for the second and third term in \eqref{proofofpropulmc1}:
\[
\begin{aligned}
\sum^d_{i=1}\phi_i \EE \mathrm{K}_i
\leq \left(-\frac{\gamma \mu h}{2}+\frac{3h^2}{\min\{\phi_i\}}\right)\EE|\Delta^m|^2\,.
\end{aligned}
\]
Substituting into \eqref{proofofpropulmc1}, we have
\[
\begin{aligned}
\EE|\Delta^{m+1}|^2\leq &\left(1-\frac{\gamma\mu h}{2}+\frac{30h^2}{\min\{\phi_i\}}+\frac{6h^2}{\min\{\phi_i\}}\right)\EE|\Delta^m|^2+\frac{6\gamma^2h^3}{\mu}\sum^d_{i=1}\frac{L^2_i}{\phi^2_i}\\
\leq &\left(1-\frac{\gamma\mu h}{4}\right)\EE|\Delta^m|^2+\frac{6\gamma^2h^3}{\mu}\sum^d_{i=1}\frac{L^2_i}{\phi^2_i}\,,
\end{aligned}
\]
where we use $h<\frac{\gamma \mu\min\{\phi_i\}}{240}$, which proves \eqref{prop:iterationulmc}.
\end{proof}

\section{Proof of Proposition \ref{prop:badexampleW22}}\label{sec:proofofprop:badexampleW22}
In this section, we prove Proposition \ref{prop:badexampleW22}.

Fisrt, we define $w^m=x^m+v^m$, and denote $u_m(x,w)$ the probability density of $(x^m,w^m)$ and $u^\ast(x,w)$ the probability density of $(x,w)$ if $(x,v=w-x)$ is distributed according to density function $p$. Recall \eqref{trivialinequlaitySAGA2}, we just need to give a lower bound for $W_2^2(u_m,u^\ast)$.

\begin{proof}[Proof of Proposition \ref{prop:badexampleW22}] 
We first notice
\begin{equation}\label{iterationW2badbound}
\begin{aligned}
W_2(u_{m},u^{\ast})\geq&\sqrt{\int |w|^2u_m(x,w)\rd w\rd x}-\sqrt{\int|w|^2u^{\ast}(x,w)\rd w\rd x}\\
=&\sqrt{\int |w|^2u_m(x,w)\rd w\rd x}-\sqrt{2d}=\sqrt{\EE|w^m|^2}-\sqrt{2d}\\
=&\frac{\EE|w^m|^2 -2d}{\sqrt{\EE|w^m|^2}+\sqrt{2d}}\,,
\end{aligned}
\end{equation}
where $\EE$ takes all randomness into account. This implies to prove \eqref{eqn:badexampleW2bound2}, it suffices to find a lower bound for second moment of $w^m$. Indeed, in the end, we will show that
\begin{equation}\label{ddd}
W_2(u_m\,,u^\ast)\geq \frac{\left(1-2h+2.9dh^2\right)^m}{800^2}\frac{d}{2}+\frac{d^{3/2}h}{80-116dh}\,,
\end{equation}
and thus
\[
W_2(q_m,p)\geq \frac{\left(1-2h+2.9dh^2\right)^m}{800^2}\frac{d}{8}+\frac{d^{3/2}h}{320-464dh}\,,
\]
proving the statement of the theorem. To show~\eqref{ddd}, we first note, that in this example:
\[
L = 1,\quad \mu=1
\]
and by direct calculation:
\begin{equation}\label{bbb}
W_2(q_0,p)=\frac{\sqrt{d}}{400}\,,\quad \EE|x^0|^2=\frac{160001}{160000}d\,,\quad \EE|\omega^0|^2=\EE|x^0|^2+\EE|v^0|^2=\frac{160001}{160000}d+d=\frac{320001}{160000}d\,,
\end{equation}
then we divide the proof into several steps:
\begin{itemize}
\item \textbf{First step:} \emph{a priori moment estimates}

According to \eqref{bbb}, use Theorem \ref{thm:rculmc} \eqref{thm:iterationulmc}, we have for any $m\geq 0$
\[
W_2(q_m,p)\leq \frac{\sqrt{d}}{100}+\frac{\sqrt{d}}{100}=\frac{\sqrt{d}}{50}\,,\quad W_2(u_m,u^*)\leq 4W_2(q_m,p)\leq\frac{2\sqrt{d}}{25}\,,
\]
Similar to \eqref{iterationW2badbound}, we have
\[
W_2(q_m,p)\geq |\sqrt{\EE|x^m|^2}-\sqrt{d}|\,,\quad W_2(u_m,u^*)\geq |\sqrt{\EE|w^m|^2}-\sqrt{2d}|
\]
which implies
\begin{equation}\label{lowerboundforx}
\sqrt{\EE|x^m|^2}\leq \frac{51\sqrt{d}}{50},\quad \sqrt{\EE|w^m|^2}\leq \left(\sqrt{2}+\frac{2}{25}\right)\sqrt{d}\,
\end{equation}
for any $m\geq 0$. 

\item \textbf{Second step:} \emph{Iteration formula of $\EE|w^m|^2$.}

By the special structure of $p$, we can calculate the second moment explicitly. Since $f(x)$ can be written as
\[
f(x)=\sum^d_{i=1}\frac{|x_i|^2}{2}\,,
\]
in each step of RC-ULMC, according to Algorithm \ref{alg:RCD-OULMC}, for each $m\geq 0$, we have for any $1\leq i\leq d$
\begin{equation}\label{eqn:EWm+1i}
\EE|w^{m+1}_i|^2=\frac{1}{d}\EE\left(|w^{m+1}_i|^2\;\middle|\;r^m=i\right)+\left(1-\frac{1}{d}\right)\EE\left(|w^{m}_i|^2\right)
\end{equation}
Under condition $r^m=i$, we have
\begin{equation}\label{distributionofZ1}
\begin{aligned}
&\EE \left(x^{m+1}_i|(x^m,v^m,r^m)\right)=x^m_i+\frac{1}{2}\left(1-e^{-2dh}\right)v^m_i-\frac{1}{2}\left(dh-\frac{1}{2}\left(1-e^{-2dh}\right)\right)x^m_i,\\
&\EE \left(v^{m+1}_i|(x^m,v^m,r^m)\right)=v^m_ie^{-2dh}-\frac{1}{2}\left(1-e^{-2dh}\right)x^m_i\,,\\
&\EE \left(w^{m+1}_i|(x^m,v^m,r^m)\right)=\frac{1}{2}\left(1+e^{-2dh}\right)w^m_i-\frac{1}{2}\left(dh-\frac{1}{2}\left(1-e^{-2dh}\right)\right)x^m_i\,,\\
&\Var\left(x^{m+1}_i|(x^m,v^m,r^m)\right)=dh-\frac{3}{4}-\frac{1}{4}e^{-4dh}+e^{-2dh}\,,\\
&\Var\left(v^{m+1}_i|(x^m,v^m,r^m)\right)=1-e^{-4dh}\,,\\
&\Cov\left((x^{m+1}_i\,,v^{m+1}_i)|(x^m,v^m,r^m)\right)=\frac{1}{2}\left[1+e^{-4dh}-2e^{-2dh}\right]\,.
\end{aligned}
\end{equation}

Now, since $dh\leq \frac{1}{10^8}$, we can replace $e^{-2dh}$ and $e^{-4dh}$ by their Taylor expansion:
\begin{equation}\label{taylorofe}
e^{-2dh}=1-2dh+2d^2h^2+D_1h^3,\quad e^{-4dh}=1-4dh+8d^2h^2+D_2h^3\,,
\end{equation}
where $D_1,D_2$ are negative constants depends on $h$ and satisfy
\[
|D_1|<10d^3,\quad |D_2|<100d^3\,.
\]
Substituting \eqref{taylorofe} into \eqref{distributionofZ1}, we have
\begin{equation}\label{distributionofZ2}
\begin{aligned}
&\EE \left(w^{m+1}_i|(x^m,w^m,r^m)\right)=\left(1-dh+d^2h^2+\frac{D_1h^3}{2}\right)w^m_i-\left(\frac{d^2h^2}{2}+\frac{D_1h^3}{4}\right)x^m_i\,,\\
&\mathrm{Var}\left(x^{m+1}_i|(x^m,v^m,r^m)\right)=\left(D_1-\frac{D_2}{4}\right)h^3\,,\\
&\Var\left(v^{m+1}_i|(x^m,v^m,r^m)\right)=4dh-8d^2h^2-D_2h^3\,,\\
&\Cov\left((x^{m+1}_i\,,v^{m+1}_i)|(x^m,v^m,r^m)\right)=2d^2h^2+\frac{\left(D_2-2D_1\right)h^3}{2}\,.
\end{aligned}
\end{equation}
The last three equalities in \eqref{distributionofZ2} implies
\[
\begin{aligned}
\Var\left(w^{m+1}_i|(x^m,v^m,r^m)\right)=&\Var\left(x^{m+1}_i|(x^m,v^m,r^m)\right)+\Var\left(v^{m+1}_i|(x^m,v^m,r^m)\right)\\
&+2\Cov\left(x^{m+1}_i\,,v^{m+1}_i)|(x^m,v^m,r^m)\right)\\
=&4dh-4d^2h^2-\left(D_1+\frac{D_2}{4}\right)h^3\,.
\end{aligned}
\]
Then, we can calculate the formula for $\EE\left(|\omega^{m+1}_i|^2\;\middle|\;r^m=i\right)$:
\[
\begin{aligned}
&\EE\left(|\omega^{m+1}_i|^2\;\middle|\;r^m=i\right)\\
=&\EE_{x^m,w^m}\left(\left.|\omega^{m+1}_i|^2\right|(x^m,v^m,r^m=i)\right)\\
=&\EE_{x^m,w^m}\left(\left|\EE \left(w^{m+1}_i|(x^m,v^m,r^m=i)\right)\right|^2+\Var\left(w^{m+1}_i|(x^m,v^m,r^m=i)\right)\right)\\
=&\left(1-dh+d^2h^2+\frac{D_1h^3}{2}\right)^2\EE|w^m_i|^2+\left(\frac{d^2h^2}{2}+\frac{D_1h^3}{4}\right)^2\EE|x^m_i|^2\\
&-2\left(1-dh+d^2h^2+\frac{D_1h^3}{2}\right)\left(\frac{d^2h^2}{2}+\frac{D_1h^3}{4}\right)\EE\left\langle w^m_i,x^m_i\right\rangle\\
&+4dh-4d^2h^2-\left(D_1+\frac{D_2}{4}\right)h^3\,.
\end{aligned}
\]
Sum them up with $i$ and use \eqref{eqn:EWm+1i}, we finally obtain an iteration formula for $\EE|w^m|^2$:
\begin{equation}\label{omegamiexpectationomega}
\begin{aligned}
\EE|\omega^{m+1}|^2=&\left[1-\frac{1}{d}+\frac{1}{d}\left(1-dh+d^2h^2+\frac{D_1h^3}{2}\right)^2\right]\EE|w^m|^2+\frac{1}{d}
\left(\frac{d^2h^2}{2}+\frac{D_1h^3}{4}\right)^2\EE|x^m|^2\\
&-\frac{2}{d}\left(1-dh+d^2h^2+\frac{D_1h^3}{2}\right)\left(\frac{d^2h^2}{2}+\frac{D_1h^3}{4}\right)\EE\left\langle w^m,x^m\right\rangle\\
&+4dh-4d^2h^2-\left(D_1+\frac{D_2}{4}\right)h^3\,.
\end{aligned}
\end{equation}

\item \textbf{Third step:} \emph{Lower bound for $W_2(u_{m},u^*)$}

Use \eqref{lowerboundforx}, since $D_1<0$, $h<\frac{1}{10^8d}<\frac{d^2}{10^7|D_1|}$, we have
\begin{equation}\label{inequalfirst}
\left(1-dh+d^2h^2+\frac{D_1h^3}{2}\right)^2\geq 1-2dh+2.9d^2h^2\,,
\end{equation}
and
\begin{equation}\label{inequalsecond}
\left|\left(\frac{d^2h^2}{2}+\frac{D_1h^3}{4}\right)\EE\left\langle w^m,x^m\right\rangle\right|\leq \frac{d^2h^2}{2}\left(\EE|w^m|^2\EE|x^m|^2\right)^{1/2}\leq 0.85d^3h^2\,.
\end{equation}

Substituting \eqref{inequalfirst} and \eqref{inequalsecond} into \eqref{omegamiexpectationomega}, we have
\begin{equation}\label{lowerboundofomegam0}
\EE|\omega^{m+1}|^2\geq\left(1-2h+2.9dh^2\right)\EE|w^m|^2+4dh-5.7d^2h^2\,.
\end{equation}
According to \eqref{bbb}, use \eqref{lowerboundofomegam0} iteratively, we have
\begin{equation}\label{lowerboundofomegam}
\begin{aligned}
\EE|\omega^{m}|^2&\geq\frac{\left(1-2h+2.9dh^2\right)^m320001}{160000}d+\left(1-(1-2h+2.9dh^2)^m\right)\frac{4d-5.7d^2h}{2-2.9dh}\\
&=\left(1-2h+2.9dh^2\right)^m\left[\frac{320001}{160000}d-\frac{4d-5.7d^2h}{2-2.9dh}\right]+\frac{4d-5.7d^2h}{2-2.9dh}\\
&\geq \left(1-2h\right)^m\frac{d}{320000}+\frac{4d-5.7d^2h}{2-2.9dh}\,,
\end{aligned}
\end{equation}
where we use $h<\frac{1}{10^8d}$ to obtain $\frac{4d-5.7d^2h}{2-2.9dh}<\frac{640001}{320000}d$ in the last inequality.

Substituting \eqref{lowerboundofomegam} into \eqref{iterationW2badbound}, we further have
\[
\begin{aligned}
W_2(u_{m},u^*)&\geq \frac{\left(1-2h\right)^m\frac{d}{320000}+\frac{4d-5.7d^2h}{2-2.9dh}-2d}{\sqrt{\left(1-2h\right)^m\frac{d}{320000}+\frac{4d-5.7d^2h}{2-2.9dh}}+\sqrt{2d}}\\
&\geq \frac{\left(1-2h\right)^m\frac{d}{320000}+\frac{0.1d^2h}{2-2.9dh}}{4\sqrt{d}}\\
&\geq \frac{\left(1-2h\right)^m}{800^2}\frac{d}{2}+\frac{d^{3/2}h}{80-116dh}\,,
\end{aligned}
\]
\end{itemize}
\end{proof}
\section{Key lemmas in the proof of Theorem \ref{lem:convergenceulmclemma}}
Consider \eqref{eqn:equationsxsw}, then we have the following lemma:
\begin{lemma}\label{lemma:thmconvergence} If $h_i\leq \frac{1}{20}$ and $\gamma\leq \frac{1}{L}$, we have
\begin{align}
&\EE\left(\sup_{t\in[T^m,T^m+h_i]}|X_i(t)-X_i(T^m)|^2\;\middle|\;\mathcal{F}^n\right)\leq 54h^2_i\left(|X^m_i|^2+|W^m_i|^2\right)+9\gamma^2h^2_i|\partial_if(X^m)|^2+540\gamma h^3_i\label{equ:boundofsupx}\\
&\EE\left(\sup_{t\in[T^m,T^m+h_i]}|W_i(t)-W_i(T^m)|^2\;\middle|\;\mathcal{F}^n\right)\leq 54h^2_i\left(|X^m_i|^2+|W^m_i|^2\right)+9\gamma^2h^2_i|\partial_if(X^m)|^2+1080\gamma h_i\label{equ:boundofsupw}\\
&\EE\left(\sup_{t\in[T^m,T^m+h_i]}|\partial_if(X(t))|^2\;\middle|\;\mathcal{F}^n\right)\leq 3|\partial_if(X^m)|^2+16L^2_ih^2\left(|X^m_i|^2+|W^m_i|^2\right)+160L_ih^3\label{equ:boundofsupf}
\end{align} 
\end{lemma}
\begin{proof}
First, we note \eqref{equ:boundofsupx},\eqref{equ:boundofsupf} are direct results from \cite{Shen2019TheRM} Lemma 6.

Consider \eqref{eqn:equationsxsw}, we have
\begin{equation}\label{eqn:lemmaA1}
\begin{aligned}
&\EE\left(\sup_{t\in[T^m,T^m+h_i]}|W_i(t)-W_i(T^m)|^2\;\middle|\;\mathcal{F}^n\right)\\
\leq&2\left(|X^m_i|^2+|W^m_i|^2\right)\\
&+\gamma^2\EE\left(\sup_{t\in[T^m,T^m+h_i]}\left|\int^{t}_{T^m}\left(1+e^{-2(t-s)}\right)\partial_{i}f(X(s)) \rd s\right|^2\;\middle|\;\mathcal{F}^n\right)\\
&+16\gamma\EE\left(\sup_{t\in[T^m,T^m+h_i]}\left|\int^{t}_{T^m}\left(1+e^{-2(t-s)}\right)\rd B_s\right|^2\;\middle|\;\mathcal{F}^n\right)\,,
\end{aligned}
\end{equation}
where the second term and third term can be further bounded by
\[
\EE\left(\sup_{t\in[T^m,T^m+h_i]}\left|\int^{t}_{T^m}\left(1+e^{-2(t-s)}\right)\partial_{i}f(X(s)) \rd s\right|^2\;\middle|\;\mathcal{F}^n\right)\leq 2h^2_i\EE\left(\sup_{t\in[T^m,T^m+h_i]}|\partial_if(X(t))|^2\;\middle|\;\mathcal{F}^n\right)
\]
and
\[
\EE\left(\sup_{t\in[T^m,T^m+h_i]}\left|\int^{t}_{T^m}\left(1+e^{-2(t-s)}\right)\rd B_s\right|^2\;\middle|\;\mathcal{F}^n\right)\leq 18h_i
\]
Substituting these into \eqref{eqn:lemmaA1}, we obtain \eqref{equ:boundofsupw}.
\end{proof}
\section{Key lemmas in the proof of Theorem \ref{thm:rculmc}}
In this section, we always assume $f$ and $h$ satisfy conditions in Theorem \ref{thm:rculmc}. And we use notations from Section \ref{sec:proofofthm:rculmc}.

\begin{lemma}\label{lem:XCulmc} For any $1\leq i\leq d$ and $m\geq 0$
\begin{equation}\label{xboundrcULMC}
\EE\left(\left.\int^{T^{m+1}}_{T^m}\left|\wrx_i(t)-\wx^m_i\right|^2\right|r^m=i\right)\rd t\leq \frac{h^3_i\gamma}{3}
\end{equation}
and
\begin{equation}\label{vboundrcULMC}
\begin{aligned}
&\EE\left(\int^{T^{m+1}}_{T^m}\left.\left|\left(\wrv_i(t)-\rv_i(t)\right)-\left(\wv^m_i-v^m_i\right)\right|^2\rd t\right| r^m=i\right)\\
\leq&\frac{8h^3_i}{3}\EE\left|\wv^m_i-v^m_i\right|^2+\frac{4\gamma^2 h^3_i}{3}\EE\left(\left|\partial_i f(\wx^m)-\partial_i f(x^m)\right|^2\right)+\frac{4\gamma^3 L^2_i h^5_i}{15}\,.
\end{aligned}
\end{equation}
\end{lemma}

\begin{lemma}\label{lem:BDrculmc} For $B^m_i,D^m_i$ defined in \eqref{eqn:BSAGA},\eqref{eqn:DSAGA}, we have
\begin{align}\label{bound:Brculmc}
\EE\left(\left.|B^m_i|^2\right| r^m=i\right) & \leq \frac{16h^4_i}{3}\EE\left|\wv^m_i-v^m_i\right|^2+\frac{8\gamma^2 h^4_i}{3}\EE\left(\left|\partial_i f(\wx^m)-\partial_i f(x^m)\right|^2\right)+\frac{6\gamma^3L^2_ih^4_i}{5}\,, \\
\label{bound:Drculmc}
\EE\left(\left.|D^m_i|^2\right| r^m=i\right)& \leq\frac{8h^4_i}{3}\EE\left|\wv^m_i-v^m_i\right|^2+\frac{4\gamma^2 h^4_i}{3}\EE\left(\left|\partial_i f(\wx^m)-\partial_i f(x^m)\right|^2\right)+\frac{4\gamma^3 L^2_i h^6_i}{15}\,.
\end{align}
\end{lemma}

\begin{lemma}\label{lem:ACrculmc}
Assume $\mathrm{K}^m_i$ is defined in \eqref{eqn:KSAGA}, then 
\begin{equation}\label{ACboundSAGA}
\begin{aligned}
\sum^d_{i=1}\phi_i \EE\mathrm{K}^m_i\leq \left(-\frac{\gamma \mu h}{2}+\frac{3h^2}{\min\{\phi_i\}}\right)\EE|\Delta^m|^2
\end{aligned}
\end{equation}
\end{lemma}

\begin{proof}[Proof of Lemma \ref{lem:XCulmc}]
First we prove \eqref{xboundrcULMC}. According to \eqref{eqn:ULDSDE2continummx}-\eqref{eqn:ULDSDE2continummxvfixed}, we have
\begin{equation}
\begin{aligned}
\EE\left(\left.\int^{T^{m+1}}_{T^m}\left|\wrx_i(t)-\wx^m_i\right|^2\right|r^m=i\right)
& =\EE\left(\left.\int^{(T^{m+1}}_{T^m}\left|\int^{t}_{T^m} \wrv_i(s)ds\right|^2\rd t\right|r^m=i\right)\\
& \leq \int^{T^m+h_i}_{T^m}(t-T^m)\int^t_{T^m} \EE\left(\left.\left|\wrv_i(s)\right|^2\right|r^m=i\right)\rd s\rd t\\
& =\int |v_i|^2p(x,v)\rd x\rd v\int^{T^m+h_i}_{T^m}(t-T^m)^2\rd t=\frac{h^3_i\gamma}{3}\,,
\end{aligned}
\end{equation}
where in the first inequality we use $T^{m+1}=T^m+h_i$ under condition $r^m=i$ and H\"older's inequality, and for the second equality we use $p$ is a stationary distribution so that $\left(\wrx_t,\wrv_t\right)\sim p$ and $\wrv_t\sim \exp(-|v|^2/(2\gamma))$ for any $t$.

Second, to prove \eqref{vboundrcULMC}, using \eqref{eqn:ULDSDE2SAGA},\eqref{eqn:ULDSDE2continummx}-\eqref{eqn:ULDSDE2continummxvfixed}, we first rewrite $\left(\wrv_i(t)-\rv_i(t)\right)-\left(\wv^m_i-v^m_i\right)$ as
\begin{equation}\label{SSAGA}
\begin{aligned}
\left(\wrv_i(t)-\rv_i(t)\right)-\left(\wv^m_i-v^m_i\right)&=\left(\wv^m_i-v^m_i\right)(e^{-2(t-T^m)}-1)\\
& \quad -\gamma\int^t_{T^m}e^{-2(t-s)}\left[\partial_i f\left(\wrx(s)\right)-\partial_i f(x^m)\right]\rd s\\
&=\mathrm{I}(t)+\mathrm{II}(t)
\end{aligned}
\end{equation}
for $T^m< t\leq T^{m+1}$. Then we bound each term separately:
\begin{equation}\label{S1SAGA}
\begin{aligned}
\EE\left(\left.\int^{T^{m+1}}_{T^m}\left|\mathrm{I}(t)\right|^2\rd t\right|r^m=i\right)&\leq h_i\EE\left(\left.\int^{T^{m+1}}_{T^m} \left|\left(\wv^m_i-v^m_i\right)(e^{-2(t-T^m)}-1)\right|^2\rd t\right|r^m=i\right)\\
&\leq h_i\int^{T^{m}+h_i}_{T^m} (2(t-T^m))^2\EE\left(\left.\left|\wv^m_i-v^m_i\right|^2\right|r^m=i\right)\rd t\\
&\leq \frac{4h^3_i}{3}\EE\left|\wv^m_i-v^m_i\right|^2\,,
\end{aligned}
\end{equation}
where we use $T^{m+1}=T^m+h_i$ and H\"older's inequality in the first inequality and $1-e^{-x}<x$ in the second inequality.
\begin{equation}\label{S2SAGA}
\begin{aligned}
&\EE\left(\left.\int^{T^{m+1}}_{T^m}\left|\mathrm{II}(t)\right|^2\rd t\right|r^m=i\right)\\
& \quad \leq \gamma^2\EE\left(\left.\int^{T^{m+1}}_{T^m}\left|\int^t_{T^m}e^{-2(t-s)}\left[\partial_i f(\wrx(s))-\partial_i f(x^m)\right]\rd s\right|^2\rd t\right|r^m=i\right)\\
& \quad \leq 2\gamma^2\EE\left(\left.\int^{T^{m+1}}_{T^m}\left|\int^t_{T^m}e^{-2(t-s)}\left[\partial_i f(\wrx(s))-\partial_i f(\wx^m)\right]\rd s\right|^2\rd t\right|r^m=i\right)\\
&\qquad +2\gamma^2\EE\left(\left.\int^{T^{m+1}}_{T^m}\left|\int^t_{T^m}e^{-2(t-s)}\left[\partial_i f(\wx^m)-\partial_i f(x^m)\right]\rd s\right|^2\rd t\right|r^m=i\right)\\
& \quad \leq 2\gamma^2\int^{T^m+h_i}_{T^m}(t-T^m)\EE\left(\left.\int^t_{T^m}\left|\partial_i f(\wrx(s))-\partial_i f(\wx^m)\right|^2\rd s\right|r^m=i\right) \rd t\\
&\qquad +2\gamma^2\int^{T^m+h_i}_{T^m}(t-T^m)\EE\left(\left.\int^t_{T^m}\left|\partial_i f(\wx^m)-\partial_i f(x^m)\right|^2\rd s\right|r^m=i\right) \rd t\\
& \quad \stackrel{(I)}{\leq} 2\gamma^2 L^2_i\int^{T^m+h_i}_{T^m}(t-T^m)\EE\left(\left.\int^t_{T^m}\left|\wrx_i(s)-\wx^m_i\right|^2\rd s\right|r^m=i\right)  \rd t\\
&\qquad +2\gamma^2\int^{T^m+h_i}_{T^m}(t-T^m)\EE\left(\int^t_{T^m}\left|\partial_i f(\wx^m)-\partial_i f(x^m)\right|^2\rd s\right) \rd t\\
& \quad \stackrel{(II)}{\leq} 2\gamma^3 L^2_i\int^{T^m+h_i}_{T^m}\frac{(t-T^m)^4}{3}\rd t+2\gamma^2 \int^{T^m+h_i}_{T^m}(t-T^m)^2\rd t\EE\left(\left|\partial_i f(\wx^m)-\partial_i f(x^m)\right|^2\right)\\
&\; \quad \leq \frac{2\gamma^3 L^2_i h^5_i}{15}+\frac{2\gamma^2 h^3_i}{3}\EE\left(\left|\partial_i f(\wx^m)-\partial_i f(x^m)\right|^2\right)\,,
\end{aligned}
\end{equation}
where in (I) we use assumption \ref{assum:Cov} \eqref{GradientLipcoord} and we use \eqref{xboundrcULMC} in (II).

Substituting \eqref{S1SAGA} and \eqref{S2SAGA} into the following inequality:
\begin{multline*}
\EE\left(\int^{T^{m+1}}_{T^m}\left.\left|\left(\wrv_i(t)-\rv_i(t)\right)-\left(\wv^m_i-v^m_i\right)\right|^2\rd t\right| r^m=i\right)\\
\leq 2\EE\left(\left.\int^{T^{m+1}}_{T^m}\left|\mathrm{I}(t)\right|^2\rd t\right|r^m=i\right)+2\EE\left(\left.\int^{T^{m+1}}_{T^m}\left|\mathrm{II}(t)\right|^2\rd t\right|r^m=i\right),
\end{multline*}
we prove \eqref{vboundrcULMC}.
\end{proof}

\begin{proof}[Proof of Lemma \ref{lem:BDrculmc}]
First, because under condition $r^m=i$, $T^{m+1}=T^m+h_i$, we can separate $B^m$ into two parts:
\[
\begin{aligned}
\EE\left(\left.|B^m_i|^2\right|r^m=i\right)\leq &2\EE\left(\left.\left|\int^{T^{m+1}}_{T^m} \left(\wrv_i(t)-\rv_i(t)\right)-\left(\wv^m
_i-v^m_i\right)\rd t\right|^2\right|r^m=i\right)\\
&+2\EE\left(\left.\left|\gamma\int^{T^{m+1}}_{T^m}e^{-2((m+1)h-t)}\left[\partial_i f(\wrx(t))-\partial_i f(\wx^m)\right]\rd t\right|^2\right|r^m=i\right)\,.
\end{aligned}
\]
We bound the two terms on the rhs:
\begin{equation}\label{V1boundSAGA}
\begin{aligned} 
&\EE\left(\left.\left|\int^{T^{m+1}}_{T^m} \left(\wrv_i(t)-\rv_i(t)\right)-\left(\wv^m_i-v^m_i\right)\rd t\right|^2\right|r^m=i\right)\\
\leq &h_i\EE\left(\left.\int^{T^{m+1}}_{T^m}\left|\left(\wrv_i(t)-\rv_i(t)\right)-\left(\wv^m_i-v^m_i\right)\right|^2\rd t\right|r^m=i\right)\\
\leq &\frac{8h^4_i}{3}\EE\left|\wv^m_i-v^m_i\right|^2+\frac{4\gamma^2 h^4_i}{3}\EE\left(\left|\partial_i f(\wx^m)-\partial_i f(x^m)\right|^2\right)+\frac{4\gamma^3 L^2_i h^6_i}{15}\,,
\end{aligned}
\end{equation}
where we use Lemma \ref{lem:XCulmc} \eqref{vboundrcULMC} in the second inequality.
\begin{equation}\label{V2boundSAGA}
\begin{aligned} 
&\EE\left(\left.\left|\gamma\int^{T^{m+1}}_{T^m}e^{-2(T^{m+1}-t)}\left[\partial_i f(\wrx(t))-\partial_i f(\wx^m)\right]\rd t\right|^2\right|r^m=i\right)\\
\leq &h_i\gamma^2\EE\left(\left.\int^{T^{m+1}}_{T^m}\left|e^{-2(T^{m+1}-t)}\left[\partial_i f(\wrx(t))-\partial_i f(\wx^m)\right]\right|^2\rd t\right|r^m=i\right)\\
\leq &h_i\gamma^2L^2_i\EE\left(\left.\int^{T^{m+1}}_{T^m}\left|\wrx_t-\wx^m\right|^2\rd t\right|r^m=i\right)\\
\leq &\frac{\gamma^3L^2_ih^4_i}{3}\,,
\end{aligned}
\end{equation}
where we use Lemma \ref{lem:XCulmc} \eqref{xboundrcULMC} in the last two inequalities. 

Combine \eqref{V1boundSAGA} and \eqref{V2boundSAGA} together, we finally have
\[
\EE\left(\left.|B^m_i|^2\right|r^m=i\right)\leq \frac{16h^4_i}{3}\EE\left|\wv^m_i-v^m_i\right|^2+\frac{8\gamma^2 h^4_i}{3}\EE\left(\left|\partial_i f(\wx^m)-\partial_i f(x^m)\right|^2\right)+\frac{8\gamma^3 L^2_i h^6_i}{15}+\frac{2\gamma^3L^2_ih^4_i}{3}\,,
\]
which implies \eqref{bound:Brculmc} if we further use $h<1$.

Next, estimation of $\EE\left(\left.|D^m_i|^2\right|r^m=i\right)$ is a direct result of \eqref{V1boundSAGA}.
\end{proof}

\begin{proof}[Proof of Lemma \ref{lem:ACrculmc}]
 Recall
\[
\begin{aligned}
\sum^d_{i=1}\phi_i \mathrm{K}^m_i=&\sum^d_{i=1} \phi_i(1-h_i-e^{-2h_i})\left[(\wx^m_i-x^m_i)-(\ww^m_i-w^m_i)\right](\ww^m_i-w^m_i)\\
&-\frac{\gamma \phi_i(1-e^{-2h_i})}{2}\sum^d_{i=1}\left[\partial_if(\wx^m_i)-\partial_if(x^m_i)\right](\ww^m_i-w^m_i)\\
&+\sum^d_{i=1} \phi_ih_i\left[(\ww^m_i-w^m_i)-(\wx^m_i-x^m_i)\right](\wx^m_i-x^m_i)\\
\end{aligned}
\]
Since $h_i<1/20$, we have
\[
|1-2h_i-e^{-2h_i}|<2h^2_i\,,
\]
which implies
\[
\begin{aligned}
\sum^d_{i=1}\phi_i\mathrm{K}^m_i=&\sum^d_{i=1} \phi_ih_i\left[(\wx^m_i-x^m_i)-(\ww^m_i-w^m_i)\right](\ww^m_i-w^m_i)-\gamma \phi_ih_i\sum^d_{i=1}\left[\partial_if(\wx^m_i)-\partial_if(x^m_i)\right](\ww^m_i-w^m_i)\\
&+\sum^d_{i=1} \phi_ih_i\left[(\ww^m_i-w^m_i)-(\wx^m_i-x^m_i)\right](\wx^m_i-x^m_i)\\
&+\sum^d_{i=1} C_i\phi_ih^2_i\left[(\wx^m_i-x^m_i)-(\ww^m_i-w^m_i)\right](\ww^m_i-w^m_i)\\
&-\frac{\gamma C_i\phi_i}{2} h^2_i\sum^d_{i=1}\left[\partial_if(\wx^m_i)-\partial_if(x^m_i)\right](\ww^m_i-w^m_i)\\
=&\mathrm{I}+\mathrm{II}
\end{aligned}
\]
where $|C_i|<2$,
\[
\begin{aligned}
\mathrm{I}=&\sum^d_{i=1} \phi_ih_i\left[(\wx^m_i-x^m_i)-(\ww^m_i-w^m_i)\right](\ww^m_i-w^m_i)-\gamma \phi_ih_i\sum^d_{i=1}\left[\partial_if(\wx^m_i)-\partial_if(x^m_i)\right](\ww^m_i-w^m_i)\\
&+\sum^d_{i=1} \phi_ih_i\left[(\ww^m_i-w^m_i)-(\wx^m_i-x^m_i)\right](\wx^m_i-x^m_i)
\end{aligned}
\]
and
\[
\mathrm{II}=\sum^d_{i=1} C_i\phi_ih^2_i\left[(\wx^m_i-x^m_i)-(\ww^m_i-w^m_i)\right](\ww^m_i-w^m_i)-\frac{\gamma C_i\phi_ih^2_i}{2} \sum^d_{i=1}\left[\partial_if(\wx^m_i)-\partial_if(x^m_i)\right](\ww^m_i-w^m_i)
\]

Let $\wx^m-x^m=a$, $\ww^m-w^m=b$, we first bound $|\mathrm{II}|$, since $\phi_ih^2_i=\frac{h^2}{\phi_i}$, we have
\begin{equation}\label{IIinequality}
\begin{aligned}
|\mathrm{II}|\leq &\frac{4h^2}{\min\{\phi_i\}}\left|\sum^d_{i=1} \left[(\wx^m_i-x^m_i)-(\ww^m_i-w^m_i)\right](\ww^m_i-w^m_i)\right|\\
&+\frac{2\gamma h^2}{\min\{\phi_i\}}\left|\sum^d_{i=1}\left[\partial_if(\wx^m_i)-\partial_if(x^m_i)\right](\ww^m_i-w^m_i)\right|\\
\leq &\frac{4h^2}{\min\{\phi_i\}}|\left\langle a-b,b\right\rangle|+\frac{2\gamma h^2}{\min\{\phi_i\}}|\left\langle\nabla f(\wx^m_i)-\nabla f(x^m_i),b\right\rangle|\\
\leq &\frac{3h^2}{\min\{\phi_i\}}\left(|a|^2+|b|^2\right)\,,
\end{aligned}
\end{equation}
where we use Assumption \ref{assum:Cov} \eqref{GradientLip} and $\gamma L\leq 1$ in the last inequality.

Next, we deal with $\mathrm{I}$, since $\phi_ih_i=h$, we have
\[
\mathrm{I}=h\left\langle a-b,b\right\rangle-\gamma h \left\langle\nabla f(\wx^m_i)-\nabla f(x^m_i),b\right\rangle+h\left\langle b-a,a\right\rangle
\]
By mean-value theorem and Assumption \ref{assum:Cov}, there exists a matrix $\mu \mathrm{I}_d\leq \mathcal{H}_f\leq L\mathrm{I}_d$ such that 
\[
\nabla f(\wx^m_i)-\nabla f(x^m_i)=\mathcal{H}_fa\,,
\]
Therefore, we have
\begin{equation}\label{eqn:Iab}
\mathrm{I}=h\left[\left\langle a-b,b\right\rangle+\left\langle b-a,a\right\rangle-\gamma  \left\langle \mathcal{H}_fa,b\right\rangle\right]=h\left(a,b\right)^\top Q\left(a,b\right)\,,
\end{equation}
where
\[
Q=\begin{bmatrix}
-\mathrm{I}_{d} & \mathrm{I}_d-\frac{\gamma \mathcal{H}_f}{2}\\
 \mathrm{I}_d-\frac{\gamma \mathcal{H}_f}{2} & -\mathrm{I}_{d}
\end{bmatrix}
\]
Calculate the eigenvalue of $Q$, we need to solve
\[
\mathrm{det}\left\{(-1-\lambda)^2I_{d}-\left(I_{d}-\frac{\gamma hH}{2}\right)^2\right\}=0\,,
\]
which implies eigenvalues $\{\lambda_j\}^K_{j=1}$ solve
\[
(-1-\lambda_j)^2-\left(1-\frac{\gamma \Lambda_j}{2}\right)^2=0\,,
\]
which implies
\[
\lambda_j=-\frac{\gamma \Lambda_j}{2}\quad\text{or}\quad \lambda_j=\frac{\gamma \Lambda_j}{2}-2
\]
since $\gamma \Lambda_j\leq \gamma\Lambda\leq 1$, we have
\[
\lambda_{\max}(Q)\leq -\frac{\gamma\mu}{2}\,.
\]
Therefore, we have
\begin{equation}\label{Iinequality}
|\mathrm{I}|\leq -\frac{\gamma\mu h}{2}\left(|a|^2+|b|^2\right)\,.
\end{equation}
Combining with \eqref{Iinequality} and \eqref{IIinequality}, we prove \eqref{ACboundSAGA}.
\end{proof}

\end{document}